\documentclass{article}

\PassOptionsToPackage{numbers, compress}{natbib}
\usepackage[final]{nips_2017}
\usepackage{float}
\usepackage[utf8]{inputenc} % allow utf-8 input
\usepackage[T1]{fontenc}    % use 8-bit T1 fonts
\usepackage{hyperref}       % hyperlinks
\usepackage{url}            % simple URL typesetting
\usepackage{booktabs}       % professional-quality tables
\usepackage{amsfonts}       % blackboard math symbols
\usepackage{nicefrac}       % compact symbols for 1/2, etc.
\usepackage{microtype}      % microtypography
\usepackage{hyperref}
\usepackage[english]{babel}
\usepackage{url}
\usepackage{mathtools}
\usepackage{amsmath}
\usepackage{amsthm}
\usepackage{amssymb}
\usepackage{wrapfig}
\usepackage{bm}
\usepackage{psfrag}
\usepackage{tikz}
\usepackage{epstopdf}
\usepackage{adjustbox}

\newcommand{\beq}[1][\vspace{0.3em}]{#1\begin{equation}}
\newcommand{\eeq}{\end{equation}}

\newcommand{\bit}{\vspace{0mm}\begin{itemize}}
\newcommand{\eit}{\vspace{0mm}\end{itemize}}
\newcommand{\ben}{\vspace{0mm}\begin{enumerate}}
\newcommand{\een}{\vspace{0mm}\end{enumerate}}
\newtheorem{theorem}{Theorem}
\newtheorem{prop}{Proposition}

\newcommand{\xv}[0]{{{\bf x}}}

\newcommand{\bb}[1]{\mathbb{#1}}

\newcommand{\mc}[1]{\mathcal{#1}}

\newcommand{\Var}{\operatorname{Var}}

\newcommand{\Lim}[1]{\raisebox{0.5ex}{\scalebox{0.8}{$\displaystyle \lim_{#1}\;$}}}

\usepackage{graphicx}
\usepackage{wrapfig}
\usepackage{subcaption}
\usepackage{placeins}
\title{A-NICE-MC: Adversarial Training for MCMC}

\author{
  Jiaming Song \\
  Stanford University\\
  \texttt{tsong@cs.stanford.edu} \\
  \And
  Shengjia Zhao \\
  Stanford University\\
  \texttt{zhaosj12@cs.stanford.edu} \\
  \And
  Stefano Ermon \\
  Stanford University\\
  \texttt{ermon@cs.stanford.edu} \\
}

\begin{document}

\maketitle

\begin{abstract}
% We propose A-NICE-MC, an efficient parametric MCMC method which bridges the gap between MCMC and deep learning. First, we propose an efficient adversarial training method to train Markov chains parametrized by implicit generative models, which can be used to approximate data distributions. Then, we leverage flexible volume preserving flows to obtain parametric kernels for MCMC. From there we can obtain bootstrap samples from our own model, which can be used to iteratively improve the quality of both the model and the samples.

% A-NICE-MC provides the first framework to train efficient MCMC proposals for domain-specific targets. Empirical results demonstrate that A-NICE-MC combines the strong guarantees in MCMC and the expressiveness in deep neural networks, and is able to significantly outperform traditional non-parametric operators. 

%Markov Chain Monte Carlo (MCMC) methods are a workhorse in probabilistic inference.  
Existing Markov Chain Monte Carlo (MCMC) methods are either based on general-purpose and domain-agnostic schemes, which can lead to slow convergence, or problem-specific proposals hand-crafted by an expert.
%We propose A-NICE-MC, a novel method to automatically train flexible parametric Markov chain kernels to produce samples with desired properties.
In this paper, we propose A-NICE-MC, a novel method to automatically design efficient Markov chain kernels tailored for a specific domain.
First, we propose an efficient likelihood-free adversarial training method to train a Markov chain and mimic a given data distribution. Then, we leverage flexible volume preserving flows to obtain parametric kernels for MCMC. Using a bootstrap approach, we show how to train efficient Markov chains to sample from a prescribed posterior distribution by iteratively improving the quality of both the model and the samples.
%A-NICE-MC provides the first framework to automatically design efficient domain-specific MCMC proposals.
Empirical results demonstrate that A-NICE-MC combines the strong guarantees of MCMC with the expressiveness of deep neural networks, and is able to significantly outperform competing methods such as Hamiltonian Monte Carlo.
\end{abstract}

\section{Introduction}
Variational inference (VI) and Monte Carlo (MC) methods are two key approaches to deal with complex probability distributions in machine learning. The former  approximates an intractable distribution %using a tractable family 
%by minimizing a divergence measure in 
by solving a variational optimization problem to minimize a divergence measure with respect to some tractable family.
%; yet its performance largely depends on the choice of model. 
The latter approximates a complex distribution using a small number of typical states, obtained by sampling ancestrally from a proposal distribution or iteratively using a suitable Markov chain (Markov Chain Monte Carlo, or MCMC).
%which is guaranteed to converge the exact posterior, yet its performance is restricted by a small number of existing nonparametric operators. 

Recent progress in deep learning %\se{wouldn't call it rep. learning. function approximators and non convex opt? end to end training?}\js{I would not say optimization. because SGMCMC benefits from optimization} 
has vastly advanced the field of variational inference. Notable examples include black-box variational inference and variational autoencoders \cite{ranganath2014black,kingma2013auto,rezende2014stochastic}, which enabled variational methods to benefit from the expressive power of deep neural networks, and adversarial training \cite{goodfellow2014generative,mohamed2016learning}, which allowed the training of new families of implicit generative models with efficient ancestral sampling.
%; and autoregressive density estimators \cite{larochelle2011neural,van2016pixel}, which can represent and learn tractable yet flexible distribution families parameterized using neural networks.
%\se{also flexible proposals for ancestral sampling, autoregressive models, vaes, etc.}
MCMC methods, on the other hand, have not benefited as much from these recent advancements. Unlike variational approaches, MCMC methods are iterative in nature and %cannot naturally leverage flexible parametric approximations.
do not naturally lend themselves to 
the use of expressive function approximators
%parametric approximations
~\cite{salimans2015markov,de2001variational}.
%First, most MCMC proposals are nonparametric in nature and does not involve learning a specific model\se{not sure i understand what this means}\js{which means for current methods there is no model to begin with?}. 
Even evaluating an existing MCMC technique is often challenging, and natural performance metrics are intractable to compute \cite{gorham2015measuring,gorham2016measuring,gorham2017measuring,ermon2014designing}. Defining an objective to improve the performance of MCMC that can be easily optimized in practice over a large parameter space is itself a difficult problem \cite{mahendran2012adaptive,boyd2004fastest}.
%There are several recent methods that address these problems: \cite{salimans2015markov} optimizes a lower bound objective combines the advantages of VI and MCMC. \js{I though the bridging the gap paper was more about training VI};  \citep{gorham2015measuring,gorham2016measuring,gorham2017measuring} uses Stein's method to improve the quality measure of biased MCMC procedures. 
%\cite{mahendran2012adaptive} employs Bayesian optimization techniques that reduce autocorrelation adaptively during inference.
%\se{cite welling mcmc variational. there is also some work on adpative mcmc to improve autocorrelation. Adaptive MCMC with Bayesian Optimization}
%\se{need to be careful about infusion, gsn, reverse flow stuff, etc. somewhat related even though not the same thing}

%\se{yeah, no need to discuss. just add the citation}
%Third, neural network proposals may be hindered by poor acceptance ratios in order to satisfy detailed balance.

%\js{should we mention stochastic gradient MCMC here? it does not seem relevant to me, but I fear the reviewers might come from that area and complain about it}\se{no need i think}
%\se{add citations to stein's method to detect convergence (lester)}

To address these limitations, we introduce A-NICE-MC, a new method for training flexible MCMC kernels, e.g., parameterized using (deep) neural networks.
%and close the gap between MCMC and representation learning.
%We propose adversarial training as an effective, likelihood-free method for training Markov chains parametrized by implicit (deep) generative models. 
Given a kernel, we view the resulting Markov Chain as an implicit generative model, i.e., one where sampling is efficient but evaluating the (marginal) likelihood is intractable. We then propose adversarial training as an effective, likelihood-free method for training a Markov chain to match a target distribution.
First, we show it can be used in a learning setting to directly approximate an (empirical) data distribution. We then use the approach 
%as an auxiliary technique
to train a Markov Chain to sample efficiently from a model prescribed by an analytic expression (e.g., a Bayesian posterior distribution), the classic use case for MCMC techniques.
%\se{can be used directly for learning to approx data distributions, and as an auxiliary inference technique to approximate a given intractable model} 
We leverage flexible volume preserving flow models \cite{dinh2014nice} and a ``bootstrap'' technique to automatically design powerful domain-specific proposals that combine the guarantees of MCMC and the expressiveness of neural networks. Finally, we propose a method that decreases autocorrelation and increases the effective sample size of the chain as training proceeds. We demonstrate that these trained operators are able to significantly outperform traditional ones, such as Hamiltonian Monte Carlo, in various domains.

\section{Notations and Problem Setup}
A sequence of continuous random variables $\{x_t\}_{t=0}^{\infty}$, $x_t \in \mathbb{R}^n$,
%%\se{need continuous $x$?. $x_t$ should depend on $\theta$} 
is drawn through the following Markov chain:
$$
x_0 \sim \pi^0 \quad\quad x_{t+1} \sim T_\theta(x_{t+1}|x_{t})
$$
where $T_\theta(\cdot|x)$ is a time-homogeneous stochastic transition kernel parametrized by $\theta \in \Theta$ and $\pi^0$ is some initial distribution for $x_0$. %\se{define $\pi_0$}. 
In particular, we assume that $T_\theta$ is defined through an implicit generative model $f_\theta(\cdot | x, v)$, where $v \sim p(v)$ is an auxiliary random variable, and $f_\theta$ is a deterministic transformation (e.g., a neural network). 
Let $\pi^t_\theta$ denote the distribution for $x_t$. If the Markov chain is both irreducible and positive recurrent, then it has an unique stationary distribution $\pi_\theta = \Lim{t \rightarrow \infty} \pi^t_\theta$. We assume that this is the case for all the parameters $\theta \in \Theta$.

Let $p_d(x)$ be a target distribution over $x \in \mathbb{R}^n$, e.g, a data distribution or an (intractable) posterior distribution in a Bayesian inference setting. Our objective is to find a $T_\theta$ such that: 
\begin{enumerate}
\item \textbf{Low bias:} The stationary distribution is close to the target distribution (minimize $|\pi_\theta - p_d|)$.
\item \textbf{Efficiency:} $\{\pi_\theta^t\}_{t=0}^{\infty}$ converges quickly (minimize $t$ such that $|\pi_\theta^t - p_d| < \delta$).
\item \textbf{Low variance:} Samples from one chain $\{x_t\}_{t=0}^{\infty}$ should be as uncorrelated as possible (minimize autocorrelation of $\{x_t\}_{t=0}^{\infty}$).
%between $\{x_t\}_{t=0}^{\infty}$ and $\{x_t\}_{t=s}^{\infty}$)
\end{enumerate}
%\se{make these desiderata more formal}

%\se{we consider two settings. first, the case where we do not have an analytic expression for $p_d(x)$, but have access to samples. then, the case where we have an analytic expression, possibly up to a normalization constant.}
We think of $\pi_\theta$ as a stochastic generative model, which can be used to efficiently produce samples with certain characteristics (specified by $p_d$), allowing for efficient Monte Carlo estimates.
We consider two settings for specifying the target distribution. The first is a \emph{learning} setting where we do not have an analytic expression for $p_d(x)$ but we have access to typical samples $\{s_i\}_{i=1}^m\sim p_d$; in the second case we have an analytic expression for $p_d(x)$, possibly up to a normalization constant, but no access to samples. The two cases are discussed in Sections \ref{sec:mgan} and \ref{sec:mcmc} respectively.

%In Section \ref{sec:mgan}, we assume that samples of $T_\theta$ are directly obtained by sampling from the implicit generative model $f_\theta$, allowing a wide range of models. In Section \ref{sec:mcmc}, we consider the special case of Markov Chain Monte Carlo (MCMC), where $T_\theta$ contains a Metropolis-Hastings step that guarantees convergence to the exact target distribution.

\section{Adversarial Training for Markov Chains}
\label{sec:mgan}
%\se{recap setting. we have access to samples}
Consider the setting where we have direct access to samples from $p_d(x)$.
Assume that the transition kernel $T_\theta(x_{t+1} | x_t)$ is the following implicit generative model:
\begin{equation}
v \sim p(v) \quad x_{t+1} = f_\theta(x_t, v) \label{eq:direct}
\end{equation}

Assuming a stationary distribution $\pi_\theta(x)$ exists, the value of $\pi_\theta(x)$ is typically intractable to compute. The marginal distribution $\pi^t_\theta(x)$ at time $t$ is also intractable, since it involves integration over all the possible paths (of length $t$) to $x$. %\se{explain. evaluating the marginal $\pi_\theta^t$ involves an integral over paths}
However, we can directly obtain samples from $\pi_\theta^t$, which will be close to $\pi_\theta$ if $t$ is large enough (assuming ergodicity). This aligns well with the idea of generative adversarial networks (GANs), a likelihood free method which only requires samples from the model. 

Generative Adversarial Network (GAN) \citep{goodfellow2014generative} is a framework for training deep generative models using a two player minimax game. A generator network $G$ generates samples by transforming a noise variable $z \sim p(z)$ into $G(z)$. A discriminator network $D(\xv)$ is trained to distinguish between ``fake'' samples from the generator and ``real'' samples from a given data distribution $p_d$. Formally, this defines the following objective (Wasserstein GAN, from \cite{arjovsky2017wasserstein})
\begin{equation}
\min_G \max_D V(D, G) = \min_G \max_D \bb{E}_{x \sim p_d}[D(x)] - \bb{E}_{z \sim p(z)} [D(G(z))] \label{eq:gan-obj}
\end{equation}

In our setting, we could assume $p_d(x)$ is the empirical distribution from the samples, and choose $z \sim \pi^0$ and let $G_\theta(z)$ be the state of the Markov Chain after $t$ steps, which is a good approximation of $\pi_\theta$ if $t$ is large enough. %\se{in this setting, we have samples from $p_d$. use them for approx the first part}
However, optimization is difficult because we do not know a reasonable $t$ in advance, and the gradient updates are expensive due to backpropagation through the entire chain.%, while having slow convergence due to high variance in the estimated gradients \se{can we justify more rigorously? evidence to back this claim?}.\js{I think this is basically the problem in BBTT, reinforce etc since we are using a single path to estimate all the paths; but this argument might not be that necessary; maybe remove this to save space}

Therefore, we propose a more efficient approximation, called \textit{Markov GAN} (MGAN):
%, with the following objective:
\begin{equation}
\min_\theta \max_D \bb{E}_{x \sim p_d}[D(x)] - \lambda \bb{E}_{\bar{x} \sim \pi_\theta^{b}}[D(\bar{x})] - (1 - \lambda) \bb{E}_{x_d \sim p_d, \bar{x} \sim T_\theta^{m}(\bar{x}|x_d)}[D(\bar{x})] \label{eq:mgan-obj}
\end{equation}
%\se{should the 2 fake terms be weighted somehow? like 0.5 and 0.5? called them z instead of x?}
where $\lambda \in (0, 1), b \in \mathbb{N}^+, m \in \mathbb{N}^+$ are hyperparameters, $\bar{x}$ denotes ``fake'' samples from the generator and $T_\theta^{m}(x|x_d)$ denotes the distribution of $x$ when the transition kernel is applied $m$ times, starting from some ``real'' sample $x_d$.

We use two types of samples from the generator for training, optimizing $\theta$ such that the samples will fool the discriminator:
\begin{enumerate}
\item Samples obtained after $b$ transitions $\bar{x} \sim \pi_\theta^{b}$, starting from $x_0 \sim \pi^0$.
\item Samples obtained after $m$ transitions, starting from a data sample $x_d \sim p_d$.% with some small random perturbation. \se{perturbation? unclear}
\end{enumerate}

Intuitively, the first condition encourages the Markov Chain to converge towards $p_d$ over relatively short runs (of length $b$).
The second condition enforces that $p_d$ is a fixed point for the transition operator.
%If we only consider this requirement, the approach would correspond to ancestral sampling in a latent variable model, as in the cases of \cite{sohl2015deep}, \cite{salimans2015markov} and \cite{bordes2016learning}.
%However, in contrast to these models, our goal is to train an iterative procedure, where 
%the quality of the samples can be improved by increasing the number of simulation steps, and
%multiple samples can be cheaply generated after the burn-in period of the chain. 
%This is accomplished by the second condition, which enforces that $p_d$ is a fixed point for the transition operator.
%enforces convergence to the stationary, where 
%each point from $p_d$ has to transition to another point on the data manifold. 
\footnote{We provide a more rigorous justification in Appendix \ref{sec:mgan-proof}.} %\se{why not state the result here?} \js{the proposition is too long...}
%The objective in Equation \ref{eq:mgan-obj} is much easier to optimize than Equation \ref{eq:gan-obj} \se{back this claim?}. 
Instead of simulating the chain until convergence, which will be especially time-consuming if the initial Markov chain takes many steps to mix, the generator would run only $(b + m)/2$ steps on average. Empirically, we observe 
better training times
%training to converge faster %\se{learning convergence? or mcmc convergence?} 
by uniformly sampling $b$ from $[1, B]$ and $m$ from $[1, M]$ respectively in each iteration, so we use $B$ and $M$ %\se{these refers to $B,M$?}
as the hyperparameters for our experiments.
%, with the advantage of estimating gradients with lower variance \se{can the variance statement be backed up somehow?} \js{if the sampled chain is shorter?}. 

\subsection{Example: Generative Model for Images}
\label{sec:images}
We experiment with a distribution $p_d$ over images, such as digits (MNIST) and faces (CelebA). In the experiments, we parametrize $f_\theta$ to have an autoencoding structure, where the auxiliary variable $v \sim \mc{N}(0, I)$ is directly added to the latent code of the network serving as a source of randomness: %\se{add more detail, explain better}. 
\begin{equation}
z = \text{encoder}_\theta(x_t) \quad z^\prime = \text{ReLU}(z + \beta v) \quad x_{t+1} = \text{decoder}_\theta(z^\prime)
\end{equation}
%\se{add theta parameters}
where $\beta$ is a hyperparameter we set to $0.1$.
While sampling is inexpensive, evaluating probabilities according to $T_\theta(\cdot | x_{t})$ is generally intractable as it would require integration over $v$.
The starting distribution $\pi_0$ is a factored Gaussian distribution with mean and standard deviation being the mean and standard deviation of the training set. We include all the details, which ares based on the DCGAN \cite{radford2015unsupervised} architecture, in Appendix \ref{sec:architecture}. All the models are trained with the gradient penalty objective for Wasserstein GANs \cite{gulrajani2017improved,arjovsky2017wasserstein}, where $\lambda = 1/3$, $B = 4$ and $M = 3$.

\begin{figure}
\centering
\includegraphics[width=\textwidth]{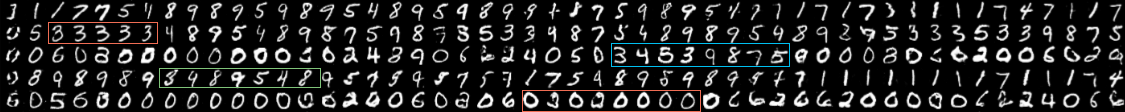}
\caption{Visualizing samples of $\pi_1$ to $\pi_{50}$ (each row) from a model trained on the MNIST dataset. Consecutive samples can be related in label (red box), inclination (green box) or width (blue box).}
\label{fig:mnist}
\end{figure}
\begin{figure}
\begin{minipage}[t]{0.25\textwidth}
\centering
\includegraphics[width=\textwidth]{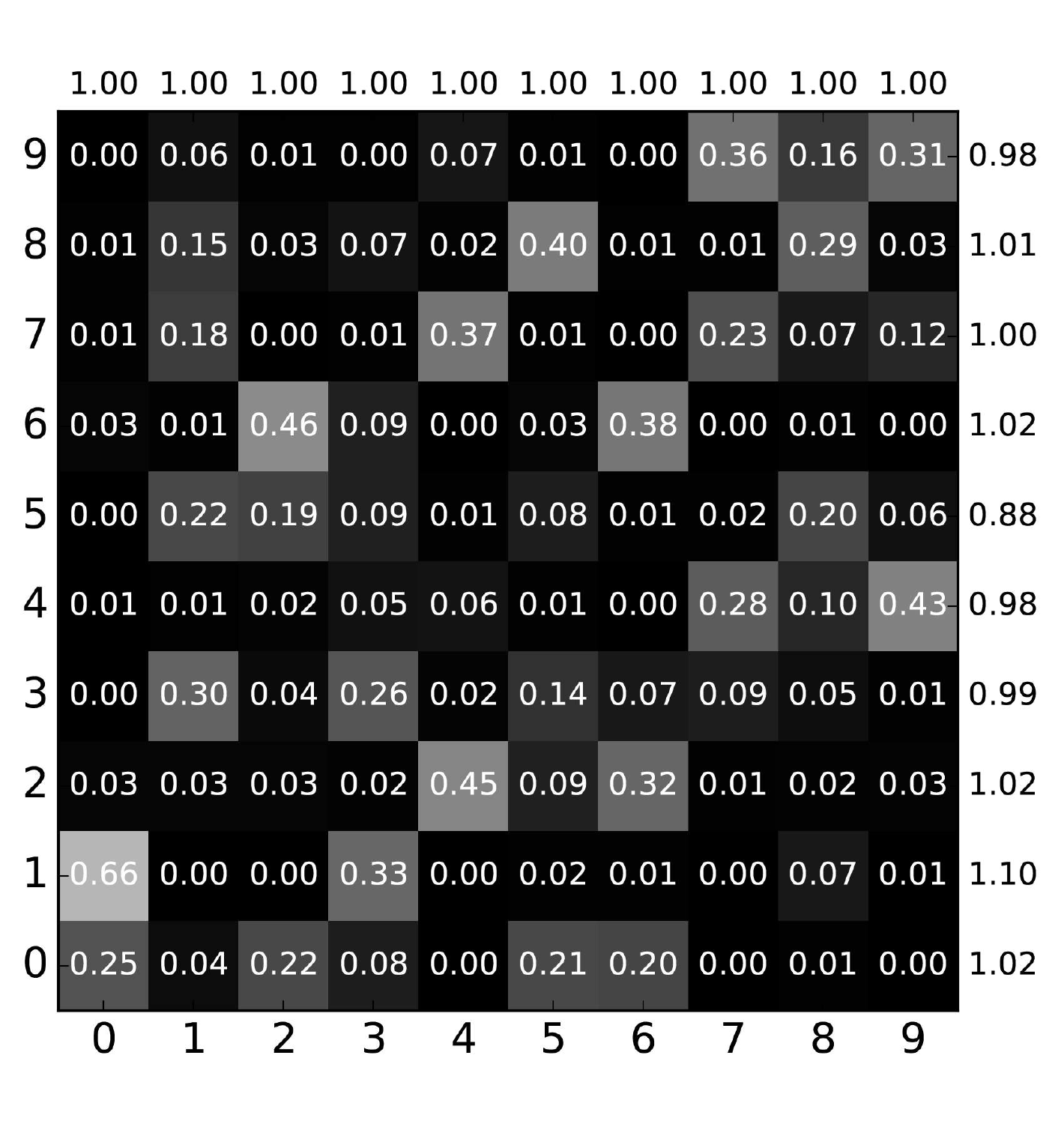}
\caption{$T_\theta(y_{t+1}|y_t)$.}
\label{fig:transition}
\end{minipage}
~
\begin{minipage}[t]{0.75\textwidth}
\centering
\includegraphics[width=\textwidth]{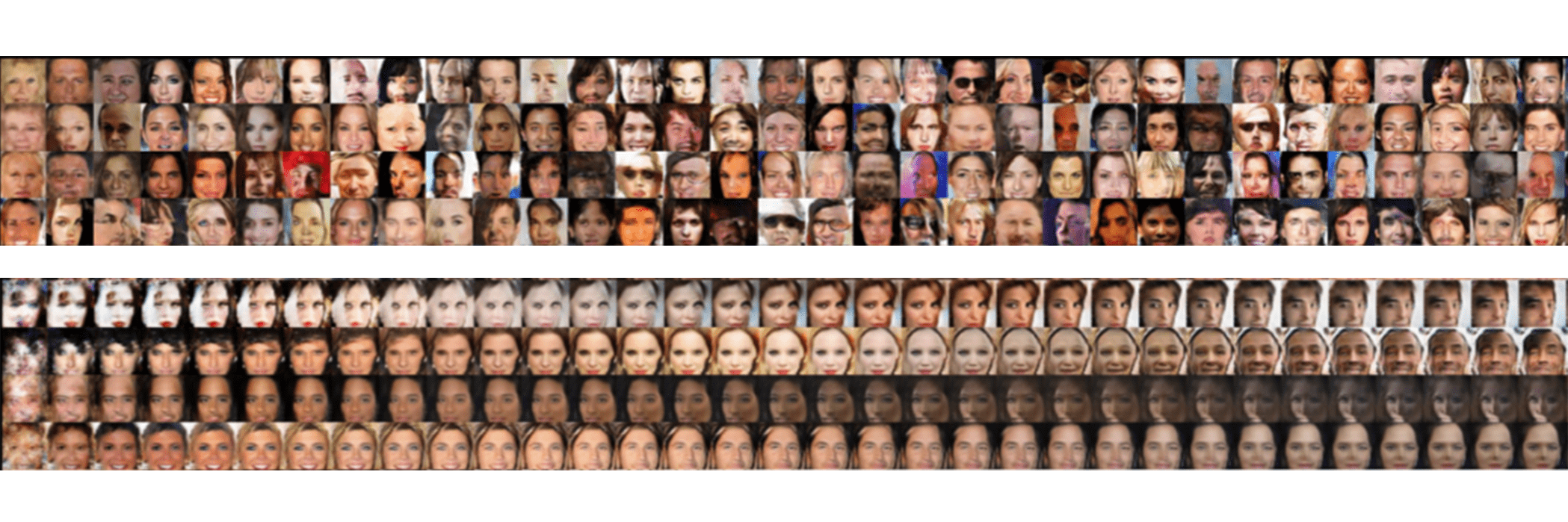}
\caption{Samples of $\pi_1$ to $\pi_{30}$ from models (top: without shortcut connections; bottom: with shortcut connections) trained on the CelebA dataset.}
\label{fig:celeba}
\end{minipage}
\end{figure}

We visualize the samples generated from our trained Markov chain in Figures \ref{fig:mnist} and \ref{fig:celeba}, where each row shows consecutive samples from the same chain (we include more images in Appendix \ref{sec:extended}) From Figure \ref{fig:mnist} it is clear that $x_{t+1}$ is related to $x_t$ in terms of high-level properties such as digit identity (label). %\se{is it possible to highlight these in the fig? using red boxes or something?}.\js{sure}
Our model learns to find and ``move between the modes'' of the dataset, instead of generating a single sample ancestrally.
This is drastically different from other iterative generative models trained with maximum likelihood, such as Generative Stochastic Networks (GSN, \cite{thibodeau2014deep}) and Infusion Training (IF, \cite{bordes2016learning}), because when we train $T_\theta(x_{t+1} | x_t)$ we are not specifying a particular target for $x_{t+1}$. In fact, to maximize the discriminator score the model (generator) may choose to generate some $x_{t+1}$ near a different mode.
%than $x_{t}$. %\se{last sentence unclear}
%\se{results are shown. describe experimental results.}

To further investigate the frequency of various modes in the stationary distribution, we consider the class-to-class transition probabilities for MNIST. We run one step of the transition operator starting from real samples where we have class labels $y \in \{0,\ldots,9\}$, and classify the generated samples with a CNN. We are thus able to quantify the transition matrix for labels in Figure \ref{fig:transition}. Results show that %although the stationary is not perfectly uniform among different classes, 
class probabilities are fairly uniform and range between $0.09$ and $0.11$. 
%\footnote{Given the transition matrix we are able to compute the stationary distribution for each class.}. %The transition probabilities indicate that detailed balance conditions are not satisfied, which is reasonable given that we do not pose any restrictions on $T_\theta$ to satisfy this condition.

%\se{might be safer not to use the word mixing. which has a technical meaning}
%\se{this part on shortcuts etc is handwavy. does not sound rigorous}\js{I think it is important to mention this because that is why we need the pairwise discriminator for MCMC}

Although it seems that the MGAN objective encourages rapid transitions between different modes, it is not always the case. In particular, as shown in Figure \ref{fig:celeba}, adding residual connections \cite{he2016deep} and highway connections \cite{srivastava2015highway} to an existing model can significantly increase the time needed to transition between modes.
%reduce mixing across modes.%, and produce sequences similar to that of Generative Stochastic Networks (GSN, \cite{thibodeau2014deep}). \js{though it seems to increase the quality of the generated images; need inception score}.
This suggests that the time needed to transition between modes can be affected by the architecture we choose for $f_\theta(x_t, v)$. If the architecture introduces an information bottleneck which forces the model to ``forget'' $x_t$, then $x_{t+1}$ will have higher chance to occur in another mode; on the other hand, if the model has shortcut connections, it tends to generate $x_{t+1}$ that are close to $x_t$. The increase in autocorrelation will hinder performance if samples are used for Monte Carlo estimates. 
%\se{mention autocorrelation.}

\section{Adversarial Training for Markov Chain Monte Carlo}
\label{sec:mcmc}
We now consider the setting where the target distribution $p_d$ is specified by an analytic expression:
\begin{equation}
\label{eq:pddef}
p_d(x) \propto \exp(-U(x))
\end{equation}
where $U(x)$ is a known ``energy function'' and the normalization constant in Equation (\ref{eq:pddef}) might be intractable to compute. This form is very common in Bayesian statistics \cite{green1995reversible}, computational physics \cite{jakob2012manifold} and graphics \cite{landau2014guide}.
Compared to the setting in Section \ref{sec:mgan}, there are two additional challenges:
\begin{enumerate}
\item We want to train a Markov chain such that the stationary distribution $\pi_\theta$ is  \emph{exactly} $p_d$;
\item We do not have direct access to samples from $p_d$ during training.
\end{enumerate}
\subsection{Exact Sampling Through MCMC}
%The first condition can be satisfied by Markov Chain Monte Carlo (MCMC) algorithms, where $\{\pi^t_\theta\}_{t=0}^{\infty}$ will asymptotically converge to the exact posterior.
We use ideas from the Markov Chain Monte Carlo (MCMC) literature to satisfy the first condition and guarantee that $\{\pi^t_\theta\}_{t=0}^{\infty}$ will asymptotically converge to $p_d$. Specifically, we require the transition operator $T_\theta(\cdot|x)$ to satisfy the \textit{detailed balance} condition:
\begin{equation}
p_d(x) T_\theta(x'|x) = p_d(x') T_\theta(x | x')
\end{equation}
for all $x$ and $x'$. This condition can be satisfied using Metropolis-Hastings (MH), where a sample $x^\prime$ is first obtained from a \emph{proposal distribution} $g_\theta(x^\prime|x)$ and accepted with the following probability:
\begin{equation}
\label{eq:mh}
A_\theta(x^\prime|x) = \min\left(1, \frac{p_d(x^\prime)}{p_d(x)} \frac{g_\theta(x | x^\prime)}{g_\theta(x^\prime | x)}\right) = \min\left(1, \exp(U(x) - U(x^\prime)) \frac{g_\theta(x | x^\prime)}{g_\theta(x^\prime | x)}\right)
\end{equation}
%\se{should be explicitly dependent on $\theta$}
Therefore, the resulting MH transition kernel can be expressed as $T_\theta(x' | x) = g_\theta(x' | x) A_\theta(x' | x)$ (if $x \neq x'$), and it can be shown that $p_d$ is stationary for $T_\theta(\cdot|x)$ \cite{hastings1970monte}.

The idea is then to optimize for a good proposal $g_\theta(x^\prime|x)$. We can set $g_\theta$ directly as in Equation (\ref{eq:direct})
%\se{not the same thing. one is a mapping the other a distribution } 
(if $f_\theta$ takes a form where the probability $g_\theta$ can be computed efficiently), %\se{don't we have to restrict to transformations where we can compute the likelihood explicitly}
 and attempt to optimize the MGAN objective in Eq. (\ref{eq:mgan-obj}) (assuming we have access to samples from $p_d$, a challenge we will address later). Unfortunately, Eq. (\ref{eq:mh}) is not differentiable - the setting is similar to policy gradient optimization in reinforcement learning.
In principle, score function gradient estimators (such as REINFORCE \cite{williams1992simple})
%, where the objective to maximize the discriminator score; 
could be used in this case; in our experiments, however, this approach leads to extremely low acceptance rates. This is because during initialization, the ratio $g_\theta(x | x^\prime)/g_\theta(x^\prime | x)$ can be extremely low, which leads to low acceptance rates and trajectories that are not informative for training. %\se{would be good to have it as a baseline. or have some experiment backing this claim.} \js{my experiments on reinforce so far has rarely reached 0.5 percent acceptance rate... I think we would mention this (and more details in the appendix) but not too much}
While it might be possible to optimize directly using more sophisticated techniques from the RL literature, we introduce an alternative approach based on volume preserving dynamics.

\subsection{Hamiltonian Monte Carlo and Volume Preserving Flow}

To gain some intuition to our method, we introduce Hamiltonian Monte Carlo (HMC) and volume preserving flow models~\cite{neal2011mcmc}.
%\js{Some transition sentence to say we draw inspirations from Hamiltonian Monte Carlo (HMC) and volume preserving flow models.} 
HMC is a widely applicable MCMC method that introduces an auxiliary ``velocity'' variable $v$ to $g_\theta(x^\prime | x)$. The proposal first draws $v$ from $p(v)$ (typically a factored Gaussian distribution) %\se{$p(v)$ undefined} 
and then obtains $(x^\prime, v^\prime)$ by simulating the dynamics (and inverting $v$ at the end of the simulation) corresponding to the Hamiltonian
\begin{equation}
H(x, v) = v^\top v/2 + U(x)
\end{equation}
where $x$ and $v$ are iteratively updated using the \textit{leapfrog integrator} (see \cite{neal2011mcmc}). The transition from $(x, v)$ to $(x^\prime, v^\prime)$ is deterministic, invertible and volume preserving, which means that 
\begin{equation}
g_\theta(x^\prime, v^\prime | x, v) = g_\theta(x, v | x^\prime, v^\prime) \label{eq:hmc-db}
\end{equation} 
MH acceptance (\ref{eq:mh}) is computed using the distribution $p(x, v) = p_d(x) p(v)$, where the acceptance probability is $p(x^\prime, v^\prime) / p(x, v)$ since $g_\theta(x^\prime, v^\prime | x, v) / g_\theta(x, v | x^\prime, v^\prime) = 1$. We can safely discard $v^\prime$ after the transition since $x$ and $v$ are independent.

Let us return to the case where the proposal is parametrized by a neural network; if we could satisfy Equation \ref{eq:hmc-db} then we could significantly improve the acceptance rate compared to the ``REINFORCE'' setting.
%, where we set $g_\theta$ according to Equation \ref{eq:direct}. 
Fortunately, we can design such an proposal by using a volume preserving flow model \cite{dinh2014nice}.

A flow model \cite{dinh2014nice,rezende2015variational,kingma2016improving,grover2017flow} defines a generative model for $x \in \bb{R}^n$ through a bijection $f: h \rightarrow x$, where $h \in \bb{R}^n$ have the same number of dimensions as $x$ with a fixed prior $p_H(h)$ (typically a factored Gaussian). In this form, $p_X(x)$ is tractable because 
\begin{equation}
p_X(x) = p_H(f^{-1}(x)) \left|\text{det} \frac{\partial f^{-1}(x)}{\partial x}\right|^{-1}
\label{eq:flow}
\end{equation}
and can be optimized by maximum likelihood.

In the case of a \textit{volume preserving flow model} $f$, the determinant of the Jacobian $\frac{\partial f(h)}{\partial h}$ is one. Such models can be constructed using \textit{additive coupling layers}, which first partition the input into two parts, $y$ and $z$, and then define a mapping from $(y, z)$ to $(y^\prime, z^\prime)$ as:
\begin{equation}
y^\prime = y \quad\quad z^\prime = z + m(y)
\end{equation}
where $m(\cdot)$ can be an expressive function, such as a neural network. 
% We can analytically derive the inverse mapping:
% \begin{equation}
% y = y^\prime \quad\quad z = z^\prime - m(y)
% \end{equation}
By stacking multiple coupling layers the model becomes highly expressive. Moreover, once we have the forward transformation $f$, the backward transformation $f^{-1}$ can be easily derived. This family of models are called \textit{Non-linear Independent Components Estimation} (NICE)\cite{dinh2014nice}.

%\se{from this description, not clear how the model is used. fix a prior on h. obtain closed form expression for the density on $x$}

\subsection{A NICE Proposal}
%\se{this proposal needs to be discussed/justified more, as it is one of the key ideas here}

HMC has two crucial components. One is the introduction of the auxiliary variable $v$, which prevents random walk behavior; the other is the symmetric proposal in Equation (\ref{eq:hmc-db}), which allows the MH step to only consider $p(x, v)$. In particular, if we 
simulate the Hamiltonian dynamics (the deterministic part of the proposal)
%apply the proposal 
twice starting from any $(x, v)$ (without MH or resampling $v$), we will always return to $(x, v)$. 

Auxiliary variables can be easily integrated into neural network proposals.  However, it is hard to obtain symmetric behavior. If our proposal is deterministic, then $f_\theta(f_\theta(x, v)) = (x, v)$ should hold for all $(x,v)$, a condition which is difficult to achieve \footnote{The cycle consistency loss (as in CycleGAN \cite{zhu2017unpaired}) introduces a regularization term for this condition; we added this to the REINFORCE objective but were not able to achieve satisfactory results.}. %\se{maybe cite cyclegan}.
%which is almost impossible for neural networks with complex architectures. 
Therefore, we introduce a proposal which satisfies Equation (\ref{eq:hmc-db}) for any $\theta$, while preventing random walk in practice by resampling $v$ after every MH step. 

Our proposal considers a NICE model $f_\theta(x, v)$ with its inverse $f_\theta^{-1}$, where $v \sim p(v)$ is the auxiliary variable. We draw a sample $x^\prime$ from the proposal $g_\theta(x^\prime, v^\prime|x, v)$ using the following procedure:
\begin{enumerate}
\item Randomly sample $v \sim p(v)$ and $u \sim \text{Uniform}[0, 1]$;
\item If $u > 0.5$, then $(x^\prime, v^\prime) = f_\theta(x, v)$;
\item If $u \leq 0.5$, then $(x^\prime, v^\prime) = f_\theta^{-1}(x, v)$.
\end{enumerate}
We call this proposal a \textit{NICE proposal} and introduce the following theorem.

% \begin{theorem}
% \label{thm:vp}
% A \textit{NICE proposal} is volume preserving.
% \end{theorem}
\begin{theorem}
\label{thm:db}
For any $(x, v)$ and $(x^\prime, v^\prime)$ in their domain, a NICE proposal $g_\theta$ satisfies 
$$
g_\theta(x^\prime, v^\prime | x, v) = g_\theta(x, v | x^\prime, v^\prime)
$$
\end{theorem}
\begin{proof}
In Appendix \ref{sec:mcmc-proof}.
\end{proof}

\subsection{Training A NICE Proposal}
% \begin{figure}
% \centering
% \includegraphics[width=0.9\textwidth]{figs/nice}
% \caption{Sampling process of A-NICE-MC ($v$ omitted). Each step, the proposal executes $f_\theta$ or $f_\theta^{-1}$. Outside the high probability regions $f_\theta$ will guide $x$ towards $p_d(x)$, while MH will tend to reject $f_\theta^{-1}$. Inside the regions both operations will have a reasonable probability of being accepted. }
% \label{fig:nice}
% \end{figure}
\begin{figure}
\centering
\includegraphics[width=0.9\textwidth]{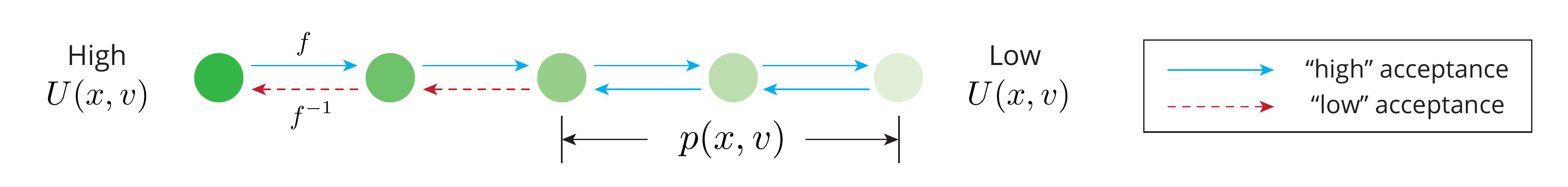}
\caption{Sampling process of A-NICE-MC. Each step, the proposal executes $f_\theta$ or $f_\theta^{-1}$. Outside the high probability regions $f_\theta$ will guide $x$ towards $p_d(x)$, while MH will tend to reject $f_\theta^{-1}$. Inside high probability regions both operations will have a reasonable probability of being accepted.}
\label{fig:nice}
\end{figure}
Given any NICE proposal with $f_\theta$, the MH acceptance step guarantees that $p_d$ is a stationary distribution, yet the ratio $p(x^\prime, v^\prime) / p(x, v)$ can still lead to low acceptance rates unless $\theta$ is carefully chosen.
%if we use some randomly initialized $\theta$, so efficient training of $f_\theta$ is required. 
Intuitively, we would like to train our proposal $g_\theta$ to produce samples that are likely under $p(x, v)$.

Although the proposal itself is non-differentiable w.r.t. $x$ and $v$, we do not require score function gradient estimators to train it. In fact, if $f_\theta$ is a bijection between samples in high probability regions, then $f_\theta^{-1}$ is automatically also such a bijection. Therefore, we ignore $f_\theta^{-1}$ during training and only train $f_\theta(x, v)$ to reach the target distribution $p(x, v) = p_d(x)p(v)$. For $p_d(x)$, we use the MGAN objective in Equation (\ref{eq:mgan-obj}); for $p(v)$, we minimize the distance between the distribution for the generated $v^\prime$ (tractable through Equation (\ref{eq:flow})) and the prior distribution $p(v)$ (which is a factored Gaussian):
\begin{equation}
\min_\theta \max_D L(x; \theta, D) + \gamma L_d(p(v), p_\theta(v^\prime))
\label{eq:mcmc-obj}
\end{equation}

where $L$ is the MGAN objective, $L_d$ is an objective that measures the divergence between two distributions and $\gamma$ is a parameter to balance between the two factors; in our experiments, we use KL divergence for $L_d$ and $\gamma = 1$ \footnote{The results are not very sensitive to changes in $\gamma$; we also tried Maximum Mean Discrepancy (MMD, see \cite{li2015generative} for details) and achieved similar results.}.

%\se{maybe expand these last two sentences. it is an important point. can say that marginally $v'$ should look like a gaussian. maybe even add equations}

%Our proposal includes a NICE network $f_\theta$ and its inverse $f_\theta^{-1}$: if we train $f_\theta$ with the MGAN objective where the target distribution is $p(x, v) = p_d(x) p(v)$
%if $f_\theta$ is a bijection that defines , then its inverse can achieve the same. We can leverage this fact and utilize the MGAN objective, where we assume $T_\theta(x^\prime, v^\prime | x, v) = f_\theta(x, v)$ and \textit{train the forward mapping only}. 

%\se{this paragraph is unclear}

%The resulting proposal can achieve high acceptance rates during inference, even though it is not directly trained. 
Our transition operator includes a trained NICE proposal followed by a Metropolis-Hastings step, and we call the resulting Markov chain \textit{Adversarial NICE Monte Carlo} (A-NICE-MC).  
The sampling process is illustrated in Figure \ref{fig:nice}.
%\se{use high prob regions instead of target manifold?}
Intuitively, if $(x, v)$ lies in a high probability region, then both $f_\theta$ and $f_\theta^{-1}$ should propose a state in another high probability region. If $(x, v)$ is in a low-probability probability region, then $f_\theta$ would move it closer to the target, while $f_\theta^{-1}$ does the opposite. However, the MH step will bias the process towards high probability regions, thereby suppressing the random-walk behavior.
%when moving towards the manifold.

\subsection{Bootstrap}

The main remaining challenge is that we do not have direct access to samples from $p_d$ in order to train $f_\theta$ according to the adversarial objective in Equation (\ref{eq:mcmc-obj}), whereas in the case of Section \ref{sec:mgan}, we have a dataset to 
get samples 
%use as samples 
from the data distribution. %\se{expand a bit. refer to previous objective} 

In order to retrieve samples from $p_d$ and train our model, we use a bootstrap process \cite{efron1994introduction} where the quality of samples 
used for adversarial training should
%from $p_d$ can 
increase over time. We obtain initial samples by running a (possibly) slow mixing operator $T_{\theta_0}$ with stationary distribution $p_d$ starting from an arbitrary initial distribution $\pi_0$. We use these samples to train our model $f_{\theta_i}$, and then use it to obtain new samples from our trained transition operator $T_{\theta_i}$; by repeating the process we can obtain samples of better quality which should in turn lead to a better model.

\subsection{Reducing Autocorrelation by Pairwise Discriminator}

\begin{figure}[t]
\centering
\begin{subfigure}{0.40\textwidth}
\includegraphics[width=\textwidth]{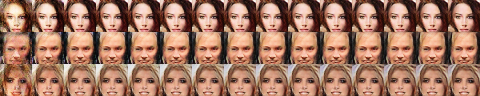}
\end{subfigure}
~
\begin{subfigure}{0.40\textwidth}
\includegraphics[width=\textwidth]{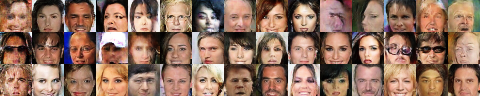}
\end{subfigure}
\caption{\textit{Left}: Samples from a model with shortcut connections trained with ordinary discriminator. \textit{Right}: Samples from the same model trained with a pairwise discriminator.}
\label{fig:pairwise}
\end{figure}
%With the presence of autocorrelation, we may obtain many samples from the MCMC algorithm, but these samples could provide much less information about the distribution than the same number of independent samples would.

An important metric for evaluating MCMC algorithms is the effective sample size (ESS), which measures the number of ``effective samples'' we obtain from running the chain. As samples from MCMC methods are not i.i.d., to have higher ESS we would like the samples to be as independent as possible (low autocorrelation).
In the case of training a NICE proposal, the objective in Equation (\ref{eq:mgan-obj}) may lead to high autocorrelation even though the acceptance rate is reasonably high. This is because the coupling layer contains residual connections from the input to the output; as shown in Section \ref{sec:images}, such models tend to learn an identity mapping and empirically they have high autocorrelation.

%\se{explain better or cut} It might be tempting to reduce autocorrelation by ``distilling'' transition kernels, where $T_\theta$ is trained to match the output after running a slow mixing kernel $T$ for $n$ times and hopefully the resulting kernel $T_\theta \approx T^n$ could have better mixing. Unfortunately, this is not always the case; \js{we show a counterexample in the appendix.}

We propose to use a \textit{pairwise discriminator} to reduce autocorrelation and improve ESS. Instead of scoring one sample at a time, the discriminator scores two samples $(x_1, x_2)$ at a time. For ``real data'' we draw two independent samples from our bootstrapped samples; for ``fake data'' we draw $x_2 \sim T_\theta^m(\cdot|x_1)$ such that $x_1$ is either drawn from the data distribution or from samples after running the chain for $b$ steps, and $x_2$ is the sample after running the chain for $m$ steps, which is similar to the samples drawn in the original MGAN objective. 

The optimal solution would be match both distributions of $x_1$ and $x_2$ to the target distribution. Moreover, if $x_1$ and $x_2$ are correlated, then the discriminator should be able distinguish the ``real'' and ``fake'' pairs, so the model is forced to generate samples with little autocorrelation. 
%Since our original MGAN objective introduces two types of chains (starting from data and starting from noise), we use pairs on these two chains to train the discriminator. 
More details are included in Appendix \ref{sec:pairwise}.
%\se{would be good to add equations for how this is done precisely, maybe in appendix}
The pairwise discriminator is conceptually similar to the minibatch discrimination layer \cite{salimans2016improved}; the difference is that we provide correlated samples as ``fake'' data, while \cite{salimans2016improved} provides independent samples that might be similar.

To demonstrate the effectiveness of the pairwise discriminator, we show an example for the image domain in Figure \ref{fig:pairwise}, %se{wrong label?}, 
where the same model with shortcut connections is trained 
with and without pairwise discrimination
%under the two different discriminators 
(details in Appendix \ref{sec:architecture}); it is clear from the variety in the samples that the pairwise discriminator significantly reduces autocorrelation.

\section{Experiments}
\label{sec:experiments}
%\subsection{Example: Energy Functions}
\label{sec:energy_fn}
Code for reproducing the experiments is available at \href{https://github.com/jiamings/a-nice-mc}{https://github.com/ermongroup/a-nice-mc}.
%\se{move to group repo to improve visibility}. 
%\js{I will move the repo to ermon-group after it is public; both links (ermongroup and jiamings) will work.}

To demonstrate the effectiveness of A-NICE-MC, we first compare its performance with HMC on several synthetic 2D energy functions: \textbf{ring} (a ring-shaped density), \textbf{mog2} (a mixture of 2 Gaussians) 
%but the two modes are distant from each 
%other, 
\textbf{mog6} (a mixture of 6 Gaussians), \textbf{ring5} (a mixture of 5 distinct rings).
%with the same width
The densities are illustrated in Figure \ref{fig:pdfs} (Appendix \ref{sec:math} has the analytic expressions). \textit{ring} has a single connected component of high-probability regions and HMC performs well; \textit{mog2}, \textit{mog6} and \textit{ring5} are selected to demonstrate cases where HMC fails to move across modes using gradient information. A-NICE-MC performs well in all the cases. 

We use the same hyperparameters for all the experiments (see Appendix \ref{sec:en} for details). In particular, we consider $f_\theta(x, v)$ with three coupling layers, which update $v$, $x$ and $v$ respectively. This is to ensure that both $x$ and $v$ could affect the updates to $x^\prime$ and $v^\prime$.

\begin{figure}
\centering
\begin{subfigure}[b]{0.13\textwidth}
\includegraphics[width=\textwidth]{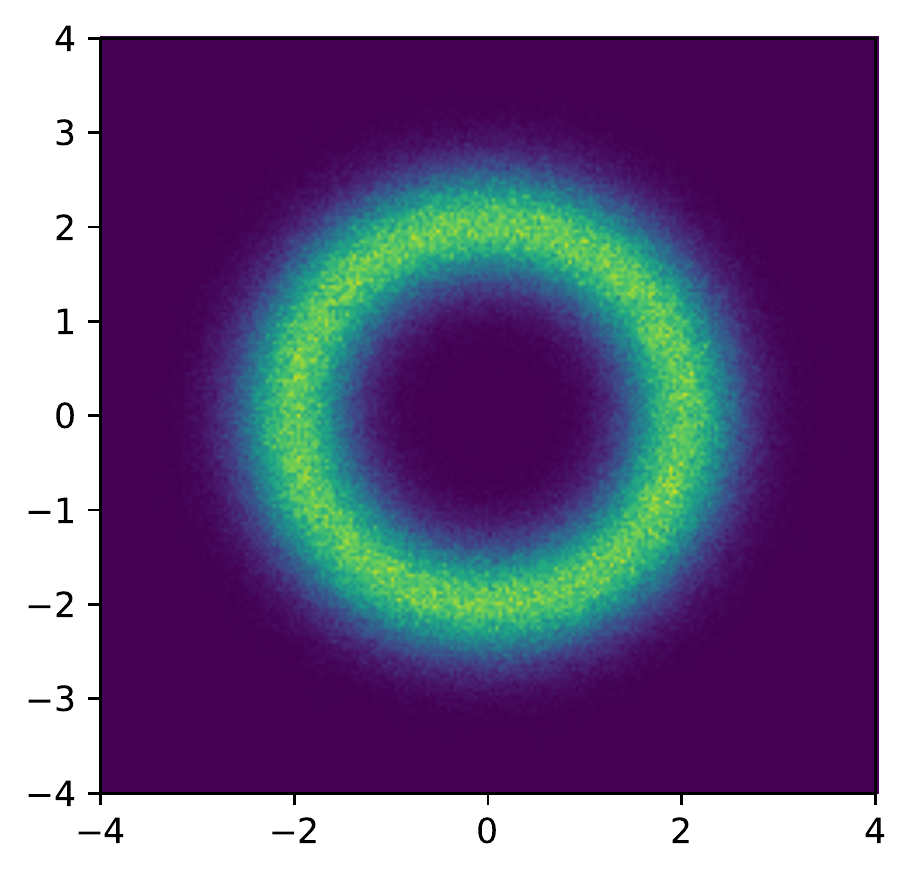}
\end{subfigure}
~
\begin{subfigure}[b]{0.13\textwidth}
\includegraphics[width=\textwidth]{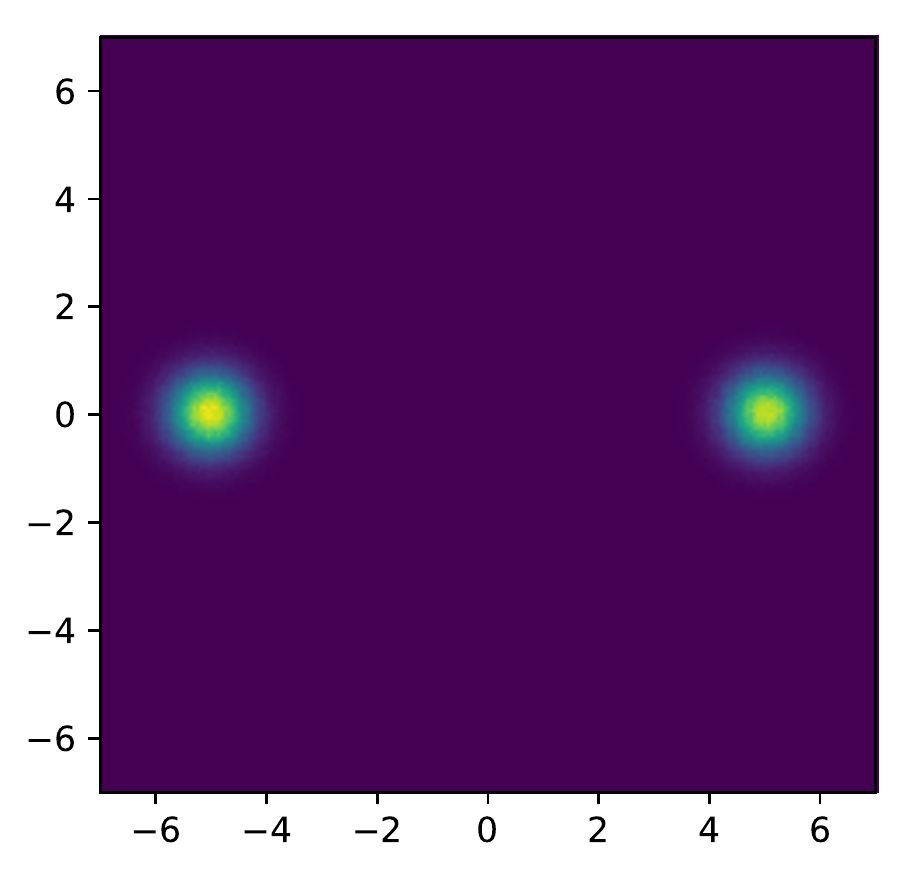}
\end{subfigure}
~
\begin{subfigure}[b]{0.13\textwidth}
\includegraphics[width=\textwidth]{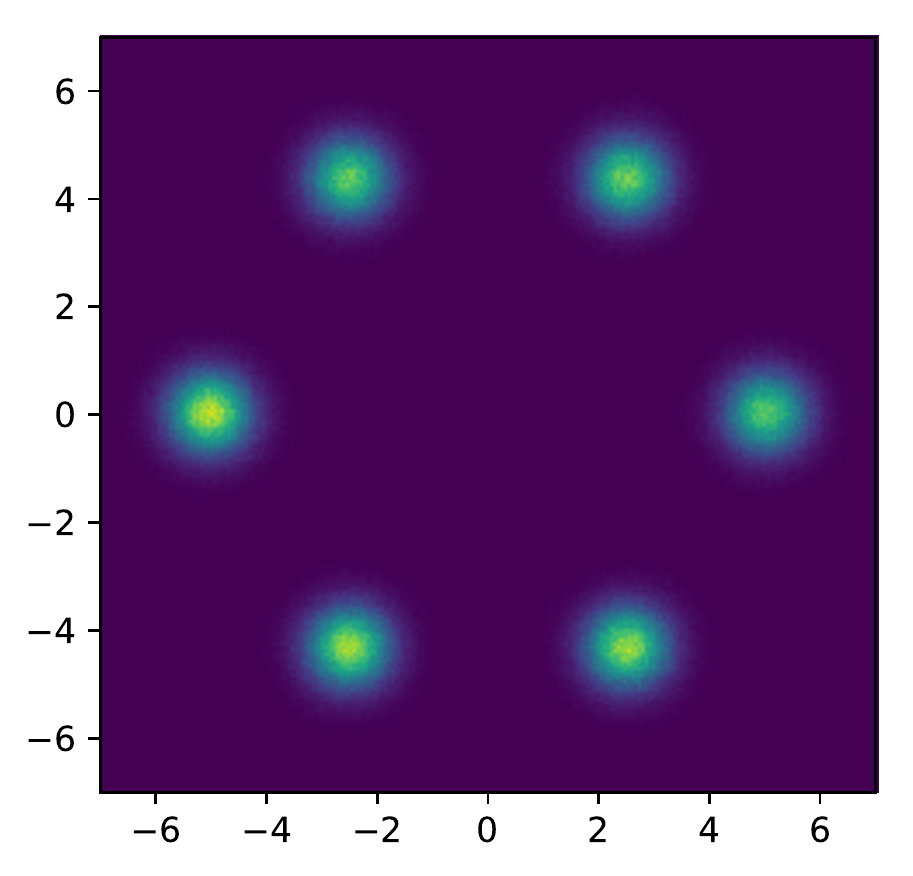}
\end{subfigure}
~
\begin{subfigure}[b]{0.13\textwidth}
\includegraphics[width=\textwidth]{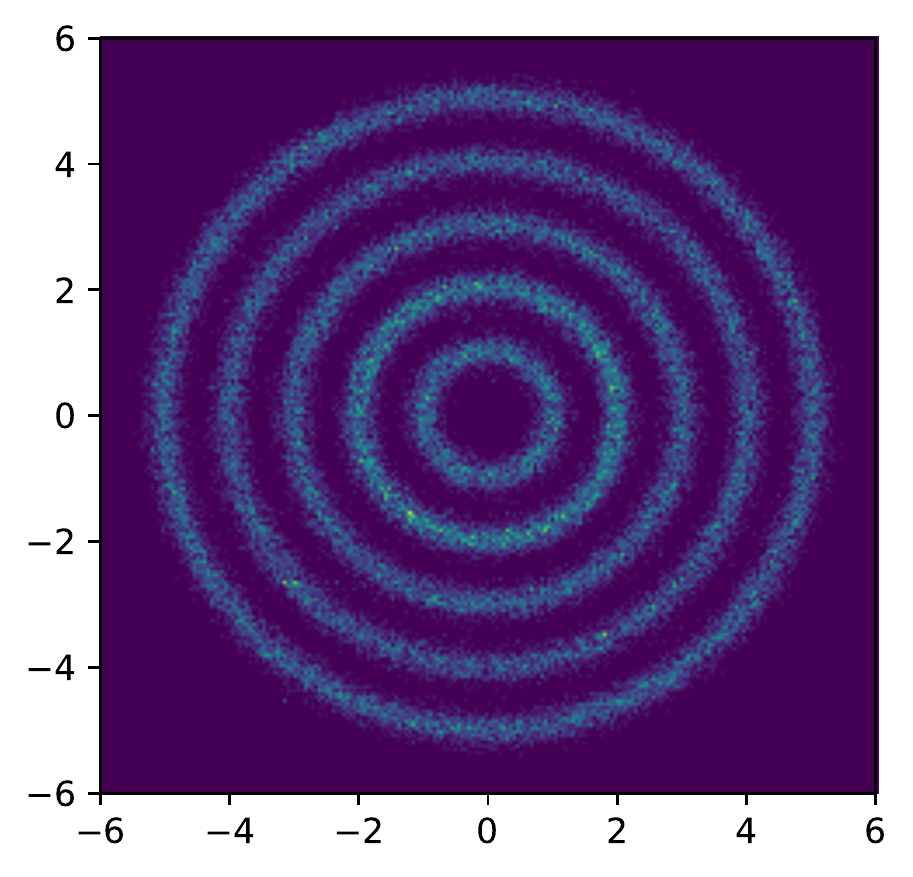}
\end{subfigure}
\caption{Densities of \textbf{ring}, \textbf{mog2}, \textbf{mog6} and \textbf{ring5} (from left to right).}
\label{fig:pdfs}
\end{figure}
%We address the following questions:

\begin{table}
\centering
\caption{Performance of MCMC samplers as measured by Effective Sample Size (ESS). Higher is better (1000 maximum). Averaged over 5 runs under different initializations. See Appendix \ref{sec:ess} for the ESS formulation, and Appendix \ref{sec:benchmark} for how we benchmark the running time of both methods.}
\label{tab:ess}
\vspace{0.5em}
\begin{subtable}{.45\textwidth}
\centering
% \begin{tabular}{ccccc}
% \toprule
% ESS & ring & mog2 & mog6 & ring5 \\\midrule
% A-NICE-MC  & & & & \\
% HMC & \textbf{1000.00} & & & \\\bottomrule
% \end{tabular}
\begin{tabular}{ccc}
\toprule
ESS & A-NICE-MC & HMC \\
\midrule
ring & \textbf{1000.00} & \textbf{1000.00}  \\
mog2 & \textbf{355.39} & 1.00  \\
mog6 & \textbf{320.03} & 1.00  \\
ring5 & \textbf{155.57} & 0.43 \\
\bottomrule
\end{tabular}
\end{subtable}
~
\begin{subtable}{.45\textwidth}
\centering
\begin{tabular}{ccc}
\toprule
ESS/s & A-NICE-MC & HMC \\
\midrule
ring & \textbf{128205} & \textbf{121212}  \\
mog2 & \textbf{50409} & 78  \\
mog6 & \textbf{40768} & 39  \\
ring5 & \textbf{19325} & 29 \\
\bottomrule
\end{tabular}
\end{subtable}
\end{table}

\paragraph{How does A-NICE-MC perform?} We evaluate and compare ESS and ESS per second (ESS/s) for both methods in Table \ref{tab:ess}. For \textit{ring}, \textit{mog2}, \textit{mog6}, we report the smallest ESS of all the dimensions (as in \cite{girolami2011riemann}); for \textit{ring5}, we report the ESS of the distance between the sample and the origin, which indicates mixing across different rings.
In the four scenarios, HMC performed well only in \textit{ring}; in cases where modes are distant from each other, there is little gradient information for HMC to move between modes. On the other hand, A-NICE-MC is able to freely move between the modes since the NICE proposal is parametrized by a flexible neural network.

\begin{figure}
\centering
\begin{subfigure}[t]{0.27\textwidth}
\includegraphics[width=\textwidth]{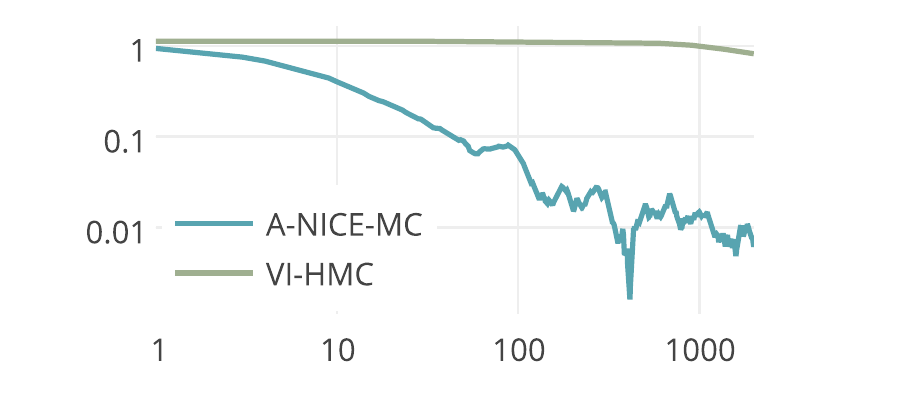}
\caption{$\mathbb{E}[\sqrt{x_1^2 + x_2^2}]$}
\label{fig:e}
\end{subfigure}
~
\begin{subfigure}[t]{0.27\textwidth}
\includegraphics[width=\textwidth]{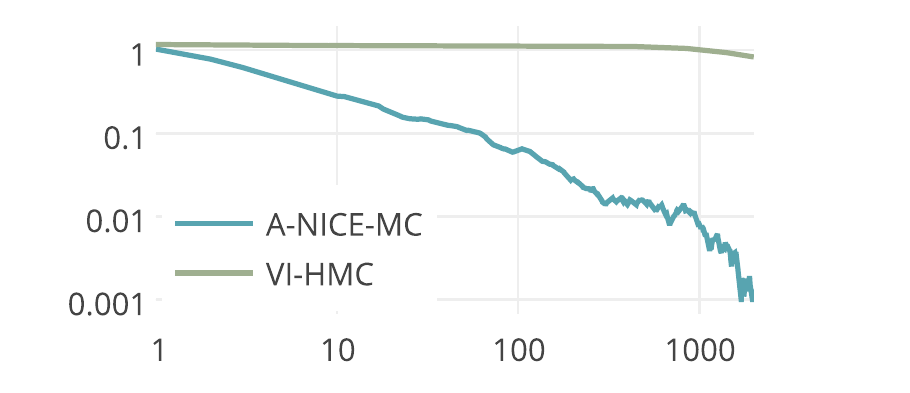}
\caption{$\text{Std}[\sqrt{x_1^2 + x_2^2}]$}
\label{fig:std}
\end{subfigure}
~
\begin{subfigure}[t]{0.18\textwidth}
\centering
\includegraphics[width=0.8\textwidth]{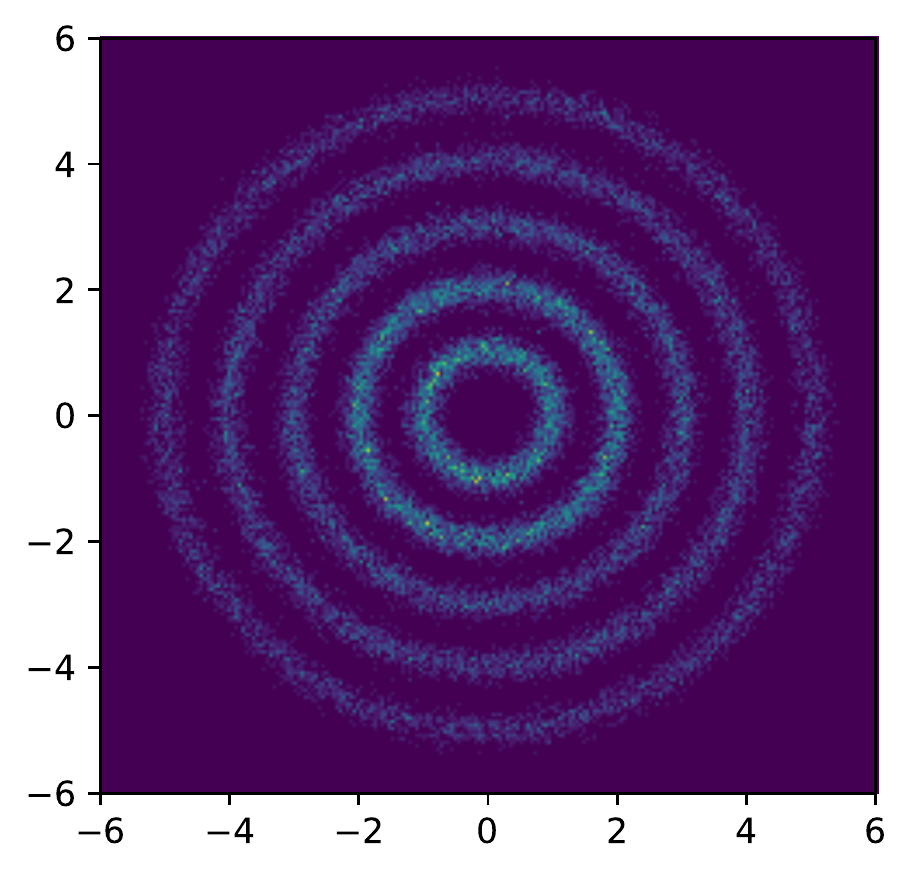}
\caption{HMC}
\label{fig:ringshmc}
\end{subfigure}
~
\begin{subfigure}[t]{0.18\textwidth}
\centering
\includegraphics[width=0.8\textwidth]{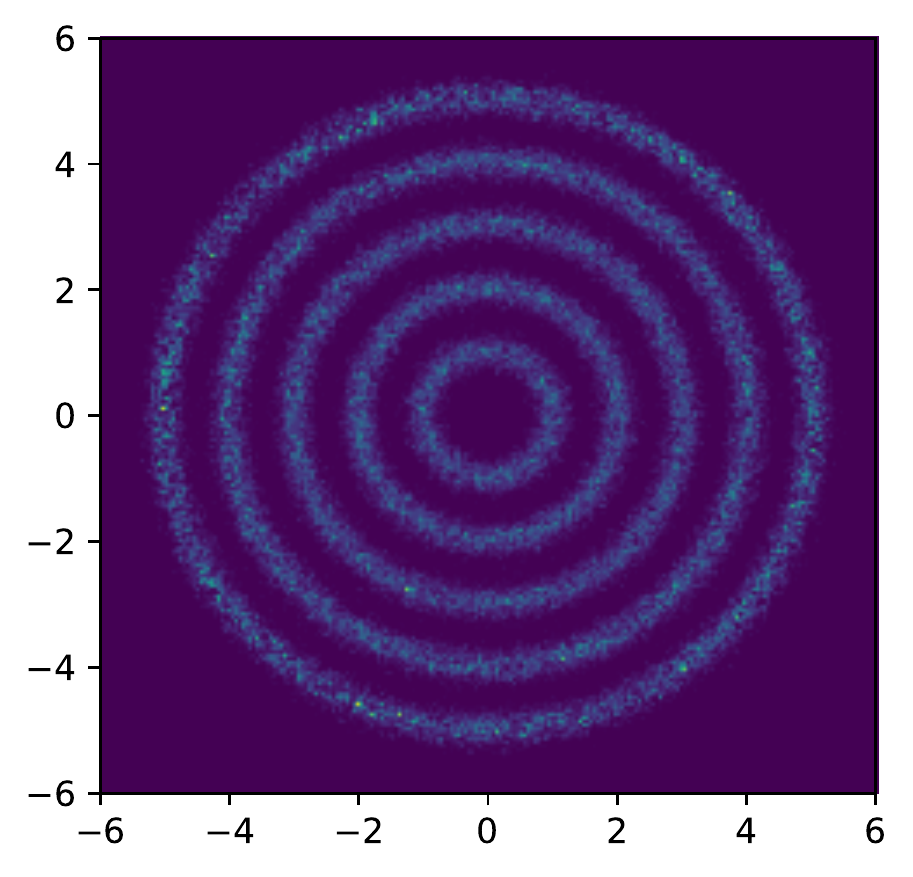}
\caption{A-NICE-MC}
\label{fig:ringsnic}
\end{subfigure}
\caption{(a-b) Mean absolute error for estimating the statistics in \textit{ring5} w.r.t. simulation length. Averaged over 100 chains. (c-d) Density plots for both methods. When the initial distribution is a Gaussian centered at the origin, HMC overestimates the densities of the rings towards the center.}
\label{fig:rings}
\end{figure}

We use \textit{ring5} as an example to demonstrate the results. We assume $\pi_0(x) = \mc{N}(0, \sigma^2 I)$ as the initial distribution, and optimize $\sigma$ through maximum likelihood. Then we run both methods, and use the resulting particles to estimate $p_d$. As shown in Figures \ref{fig:e} and \ref{fig:std}, %\se{refer to subfigure}, 
HMC fails 
%to estimate $p_d$ efficiently, where
and there is a large gap between true and estimated statistics. This also explains why the ESS is lower than 1 for HMC for \textit{ring5} in Table \ref{tab:ess}.

Another reasonable measurement to consider is Gelman’s R hat diagnostic~\citep{brooks1998general}, which evaluates performance across multiple sampled chains. We evaluate this over the rings5 domain (where the statistics is the distance to the origin), using 32 chains with 5000 samples and 1000 burn-in steps for each sample. HMC gives a R hat value of 1.26, whereas A-NICE-MC gives a R hat value of 1.002~\footnote{For R hat values, the perfect value is 1, and 1.1-1.2 would be regarded as too high.}. This suggest that even with 32 chains, HMC does not succeed at estimating the distribution reasonably well.

%\js{need to show MSE for estimating moments}

\begin{figure}[t]
\centering
\begin{subfigure}[t]{0.23\textwidth}
\includegraphics[width=\textwidth]{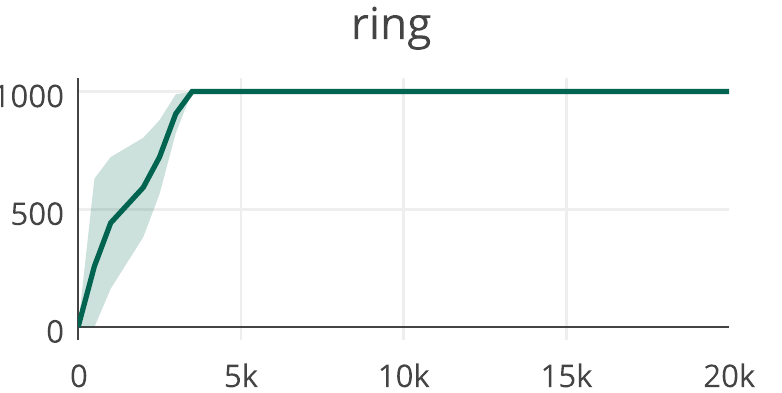}
\end{subfigure}
~
\begin{subfigure}[t]{0.23\textwidth}
\includegraphics[width=\textwidth]{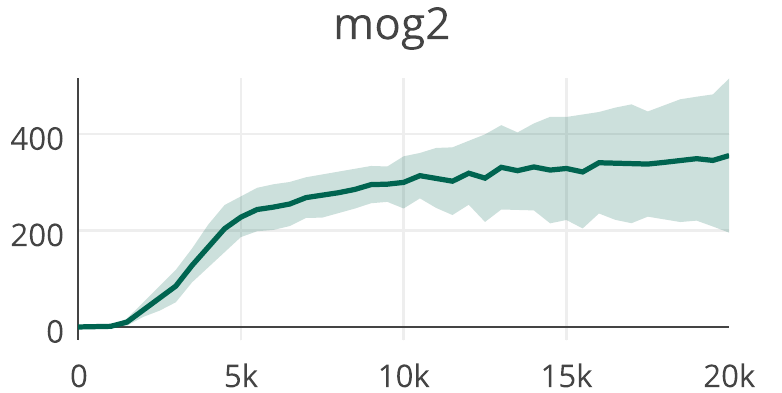}
\end{subfigure}
~
\begin{subfigure}[t]{0.23\textwidth}
\includegraphics[width=\textwidth]{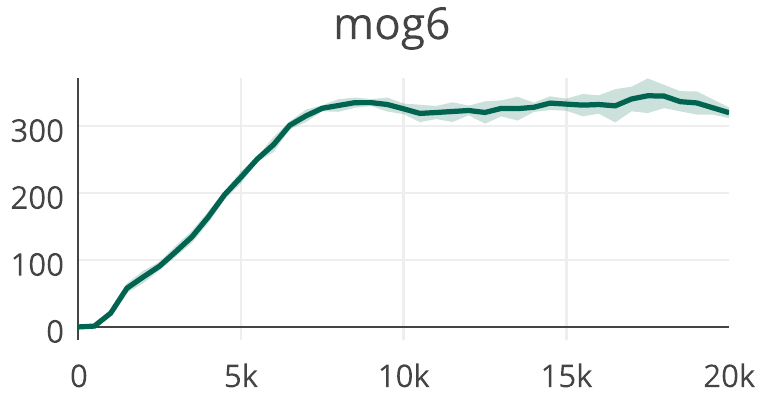}
\end{subfigure}
~
\begin{subfigure}[t]{0.23\textwidth}
\includegraphics[width=\textwidth]{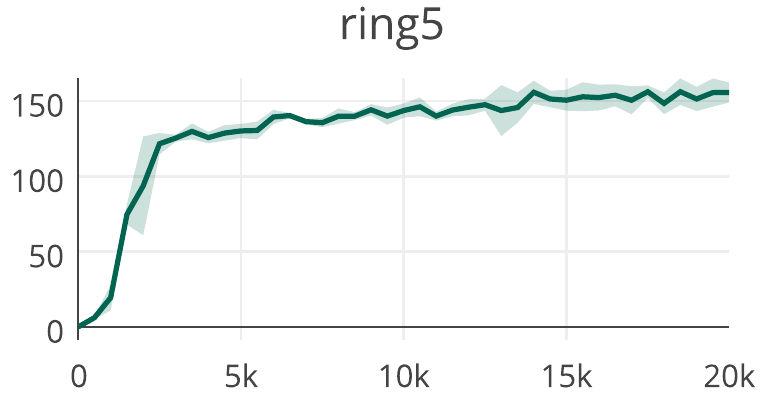}
\end{subfigure}
\caption{ESS with respect to the number of training iterations.}
\label{fig:ess}
\end{figure}

\paragraph{Does training increase ESS?} We show in Figure \ref{fig:ess} that in all cases ESS increases with more training iterations and bootstrap rounds, %\se{more bootstrap rounds?}, 
which also indicates that using the pairwise discriminator is effective at reducing autocorrelation. 

Admittedly, training introduces an additional computational cost which HMC could utilize to obtain more samples initially (not taking parameter tuning into account), yet the initial cost can be amortized thanks to the improved ESS. For example, in the \textit{ring5} domain, we can reach an ESS of $121.54$ in approximately $550$ seconds (2500 iterations on 1 thread CPU, bootstrap included). If we then sample from the trained A-NICE-MC, it will catch up with HMC in less than 2 seconds.

%\se{need to add some discsussion on how long it takes to train. whether it pays off to train and run vs run the hmc chain. based on ESS, training pays off after XXX samples.}

%\paragraph{Why not start from samples of an approximation and then run HMC?}
%\se{this part is unclear} \js{shall we raise this question in the first place?}
% We use \textit{ring5} as an example to show that the estimation error of variational approximation could affect the quality of HMC samples. In particular, if the approximation assigns incorrect densities for multiple modes, it would be difficult for HMC to travel between modes and fix the error . 

% We assume $\pi_0(x) = \mc{N}(0, \sigma^2 I)$ as the variational approximation for \textit{ring5}, and learn the $\sigma$ parameter through maximum likelihood \footnote{$\sigma \approx 1.95$. We simulate a chain with a Gaussian proposal to obtain the samples.}. Then we run HMC over the samples from $\pi_0$, and use the resulting particles to estimate $p_d$ (we call this method VI-HMC). As shown in Figure \ref{fig:rings}, VI-HMC estimates $p_d$ poorly. This also explains why the ESS is lower than 1 for HMC in \textit{ring5}.

%\subsection{Bayesian Logistic Regression}

Next, we demonstrate the effectiveness of A-NICE-MC on Bayesian logistic regression, where the posterior has a single mode in a higher dimensional space, making HMC a strong candidate for the task. However, in order to achieve high ESS, HMC samplers typically use many leap frog steps and require gradients at every step, which is inefficient when $\nabla_x U(x)$ is computationally expensive. A-NICE-MC only requires running $f_\theta$ or $f_\theta^{-1}$ once to obtain a proposal, which is much cheaper computationally.
We consider three datasets - \textit{german} (25 covariates, 1000 data points), \textit{heart} (14 covariates, 532 data points) and \textit{australian} (15 covariates, 690 data points) - and evaluate the lowest ESS across all covariates (following the settings in \cite{girolami2011riemann}), where we obtain 5000 samples after 1000 burn-in samples. For HMC we use 40 leap frog steps and tune the step size for the best ESS possible. For A-NICE-MC we use the same hyperparameters for all experiments (details in Appendix \ref{sec:blr}). 
Although HMC outperforms A-NICE-MC in terms of ESS, the NICE proposal is less expensive to compute than the HMC proposal by almost an order of magnitude, which leads to higher ESS \emph{per second} (see Table \ref{tab:ess-lr}).% \se{see table xxx}

\begin{table}
\centering
\caption{ESS and ESS per second for Bayesian logistic regression tasks.}
\label{tab:lr}
\vspace{0.5em}
\begin{subtable}{.45\textwidth}
\centering
\begin{tabular}{ccc}
\toprule
ESS & A-NICE-MC & HMC \\
\midrule
german & 926.49 & \textbf{2178.00}  \\
heart & 1251.16 & \textbf{5000.00}  \\
australian & 1015.75 & \textbf{1345.82} \\
\bottomrule
\end{tabular}
\end{subtable}
~
\begin{subtable}{.45\textwidth}
\centering
\begin{tabular}{ccc}
\toprule
ESS/s & A-NICE-MC & HMC \\
\midrule
german & \textbf{1289.03} & 216.17  \\
heart & \textbf{3204.00} & 1005.03  \\
australian & \textbf{1857.37} & 289.11 \\
\bottomrule
\end{tabular}
\end{subtable}
\label{tab:ess-lr}
\end{table}

\section{Discussion}
To the best of our knowledge, this paper presents the first likelihood-free method to train a parametric MCMC operator with good mixing properties. The resulting Markov Chains can be used to target both empirical and analytic distributions.
We showed that using our novel training objective we can leverage flexible neural networks and volume preserving flow models to obtain domain-specific transition kernels. %We show empirically that t
These kernels significantly outperform traditional ones which are based on elegant yet very simple and general-purpose analytic formulas. Our hope is that these ideas will allow us to bridge the gap between MCMC and neural network function approximators, similarly to what ``black-box techniques'' did in the context of variational inference \cite{ranganath2014black}.
%thus is not restricted by the local information of the curvature.

Combining the guarantees of MCMC and the expressiveness of neural networks unlocks the potential to perform fast and accurate inference in high-dimensional domains, such as Bayesian neural networks. This would likely require us to gather the initial samples through other methods, such as variational inference, since the chances for untrained proposals to ``stumble upon'' low energy regions is diminished by the curse of dimensionality. 
Therefore, it would be interesting to see whether we could bypass the bootstrap process and directly train on $U(x)$ by leveraging the properties of flow models.
Another promising future direction is to investigate proposals that can rapidly adapt to changes in the data. One use case is to infer the latent variable of a particular data point, as in variational autoencoders. We believe it should be possible to utilize meta-learning algorithms with data-dependent parametrized proposals.

\section*{Acknowledgements}
This research was funded by Intel Corporation, TRI, FLI and NSF grants 1651565, 1522054, 1733686. The authors would like to thank Daniel L\'{e}vy for discussions on the NICE proposal proof, Yingzhen Li for suggestions on the training procedure and Aditya Grover for suggestions on the implementation.
\FloatBarrier

\bibliographystyle{ieeetr}
\bibliography{references}

\newpage
\appendix

\section{Estimating Effective Sample Size}
\label{sec:ess}
Assume a target distribution $p(x)$, and a Markov chain Monte Carlo (MCMC) sampler that produces a set of N correlated samples $\{x_i\}_{1}^{N}$ from some distribution $q(\{x_i\}_{1}^{N})$ such that $q(x_i) = p(x_i)$. Suppose we are estimating the mean of $p(x)$ through sampling; we assume that increasing the number of samples will reduce the variance of that estimate.

Let $V = \Var_q[\sum_{i=1}^{N} x_i/N]$ be the variance of the mean estimate through the MCMC samples. The effective sample size (ESS) of $\{x_i\}_{1}^{N}$, which we denote as $M = ESS(\{x_i\}_{1}^{N})$, is the number of independent samples from $p(x)$ needed in order to achieve the same variance, i.e. $\Var_p[\sum_{j=1}^{M} x_j / M] = V$. A practical algorithm to compute the ESS given $\{x_i\}_{1}^{N}$ is provided by:

% The effective sample size (ESS) of $\{x_i\}_{1}^{N}$ is the number of independent samples needed from $p(x)$ has equal variance to the MCMC estimate of the mean of $x$.

%\se{rephrase?not sure it makes sense?variance=mean? what is the difference between monte carlo and MCMC??}\js{I basically copied this from No U turn sampler paper}:

\begin{equation}
ESS(\{x_i\}_{1}^{N}) = \frac{N}{1 + 2 \sum_{s=1}^{N-1}(1 - \frac{s}{N})\rho_s}
\end{equation}
where $\rho_s$ denotes the autocorrelation under $q$ of $x$ at lag $s$. We compute the following empirical estimate $\hat{\rho}_s$ for $\rho_s$:
\begin{equation}
\hat{\rho}_s = \frac{1}{\hat{\sigma}^2 (N - s)} \sum_{n = s+1}^{N} (x_n - \hat{\mu}) (x_{n-s} - \hat{\mu})
\end{equation}
where $\hat{\mu}$ and $\hat{\sigma}$ are the empirical mean and variance obtained by an independent sampler. 

Due to the noise in large lags $s$, we adopt the approach of \cite{hoffman2014no} where we truncate the sum over the autocorrelations when the autocorrelation goes below 0.05.

\section{Justifications for Objective in Equation \ref{eq:mgan-obj}}
\label{sec:mgan-proof}
We consider two necessary conditions for $p_d$ to be the stationary distribution of the Markov chain, which can be translated into a new algorithm with better optimization properties, described in Equation \ref{eq:mgan-obj}.
\newcommand{\TV}{\text{TV}}
\begin{prop}
\label{prop:mgan}
Consider a sequence of ergodic Markov chains 
%defined for a sequence of random variables $\{\Xv_t\}_{t=0}^{\infty}$ 
over state space $\mc{S}$. Define $\pi_n$ as the stationary distribution for the $n$-th Markov chain, and $\pi_n^t$ as the probability distribution at time step $t$ for the $n$-th chain. If the following two conditions hold:
\begin{enumerate}
\item $\exists b > 0$ such that the sequence $\{\pi^{b}_n\}_{n=1}^{\infty}$ converges to $p_d$ in total variation; %\stefano{this should be changed to a limit? this equality condition never really happens i think}
\item $\exists \epsilon > 0$, $\rho<1$ such that $\exists M > 0, \forall n > M$, if $\lVert \pi_n^t - p_d \lVert_{\TV} < \epsilon$, then $\lVert\pi_n^{t+1} - p_d\lVert_{\TV} < \rho \Vert \pi_n^{t} - p_d \lVert_{\TV} \ $; % for some $\rho < 1$;
%a\se{$\pi^t$ not defined anywhere}
\end{enumerate}
then the sequence of stationary distributions $\{\pi_n\}_{n=1}^{\infty}$ converges to $p_d$ in total variation.
\end{prop}
\begin{proof}
%\se{delta should be epsilon}
The goal is to prove that $\forall \delta > 0$, $\exists K > 0, T > 0$, such that $\forall n > N, t > T$, $\lVert \pi_n^t - p_d \lVert_{\TV} < \delta$.

%\begin{proof}

According to the first assumption, $\exists N > 0$, such that $\forall n > N$, $\lVert\pi_n^{b} - p_d\lVert_{\TV} < \epsilon$.

Therefore, $\forall n > K = \max(N, M)$, %\se{missing quantification over $n$. for all $n>N$ i think} 
$\forall \delta > 0$, $\exists T = b + \max(0, \lceil\log_\rho \delta - \log_\rho \epsilon\rceil) + 1$, % \se{t bar is undefined}
such that $\forall t > T$, 
\begin{align}
&\ \lVert\pi_n^{t} - p_d\lVert_{\TV} \nonumber \\
= & \lVert\pi_n^{b} - p_d\rVert_{\TV} \prod_{i=b}^{t-1} \frac{\lVert\pi_n^{i+1} - p_d\lVert_{\TV}}{\lVert\pi_n^{i} - p_d\lVert_{\TV}} \nonumber \\
< &\ \epsilon \rho^{t - b} < \epsilon \rho^{T - b} < \epsilon \cdot \frac{\delta}{\epsilon} = \delta
\end{align}
%\se{this is very confusing. please add more details. when do you use assumption 2? is the epsilon the same?}

The first inequality uses the fact that $\lVert\pi_n^{b} - p_d\lVert_{\TV} < \epsilon$ (from Assumption 1), and $\lVert\pi_n^{t+1} - p_d\lVert_{\TV} / \Vert \pi_n^{t} - p_d \lVert_{\TV} < \rho$ (from Assumption 2). The second inequality is true because $\rho < 1$ by Assumption 2. The third inequality uses the fact that $T - b > \lceil\log_\rho \delta - \log_\rho \epsilon\rceil$ (from definition of $T$), so $\rho^{T - b} < \delta / \epsilon$.
Hence the sequence $\{\pi_n\}_{n=1}^{\infty}$ converges to $p_d$ in total variation.
\end{proof}

Moreover, convergence in total variation distance is equivalent to convergence in Jensen-Shannon (JS) divergence\citep{arjovsky2017wasserstein}, which is what GANs attempt to minimize \citep{goodfellow2014generative}. This motivates the use of GANs to achieve the two conditions in Proposition \ref{prop:mgan}.
This suggests a new optimization criterion, where we look for a $\theta$ that satisfies both conditions in Proposition \ref{prop:mgan}, which translates to Equation \ref{eq:mgan-obj}.

\section{Proof of Theorem \ref{thm:db}}
\label{sec:mcmc-proof}
\begin{proof}
For any $(x, v)$ and $(x^\prime, v^\prime)$, $g$ satisfies:
\begin{align}
g(x^\prime, v^\prime | x, v) & = \frac{1}{2} \Big|\text{det} \frac{\partial f(x, v)}{\partial (x, v)} \Big|^{-1} \bb{I}(x^\prime, v^\prime = f(x, v)) + \frac{1}{2} \Big|\text{det} \frac{\partial f(x, v)}{\partial (x, v)} \Big| \bb{I}(x^\prime, v^\prime = f^{-1}(x, v)) \nonumber \\
& = \frac{1}{2} \bb{I}(x^\prime, v^\prime = f(x, v)) + \frac{1}{2}  \bb{I}(x^\prime, v^\prime = f^{-1}(x, v)) \nonumber \\
& = \frac{1}{2} \bb{I}(x, v = f^{-1}(x^\prime, v^\prime)) + \frac{1}{2}  \bb{I}(x, v = f(x^\prime, v^\prime)) \nonumber \\
& = g(x, v | x^\prime, v^\prime)
% \begin{cases}
%     0.5       & \quad \text{if } x^\prime, v^\prime = f(x, v)\\
%     0.5       & \quad \text{if } x^\prime, v^\prime = f^{-1}(x, v) \\
% 	0         & \quad \text{otherwise}
% \end{cases}
\end{align}
%\se{define the indicator I. maybe add more details in the last equality, showing the definition with the determinants =1}
where $\mathbb{I}(\cdot)$ is the indicator function, the first equality is the definition of $g(x^\prime, v^\prime | x, v)$, the second equality is true since $f(x, v)$ is volume preserving, the third equality is a reparametrization of the conditions, and the last equality uses the definition of $g(x, v | x^\prime, v^\prime)$ and $f$ is volume preserving, so the determinant of the Jacobian is 1.
\end{proof}

Theorem \ref{thm:db} allows us to use the ration $p(x^\prime, v^\prime) / p(x, v)$ when performing the MH step.

\section{Details on the Pairwise Discriminator}
\label{sec:pairwise}
Similar to the settings in MGAN objective, we consider two chains to obtain samples:
\begin{itemize}
\item Starting from a data point $x$, sample $z_1$ in $B$ steps.
\item Starting from some noise $z$, sample $z_2$ in $B$ steps; and from $z_2$ sample $z_3$ in $M$ steps.
\end{itemize}

For the ``generated'' (fake) data, we use two type of pairs $(x, z_1)$, and $(z_2, z_3)$. This is illustrated in Figure \ref{fig:pairwise-demo}. We assume equal weights between the two types of pairs.

\begin{figure}[h]
\centering
\includegraphics[width=\textwidth]{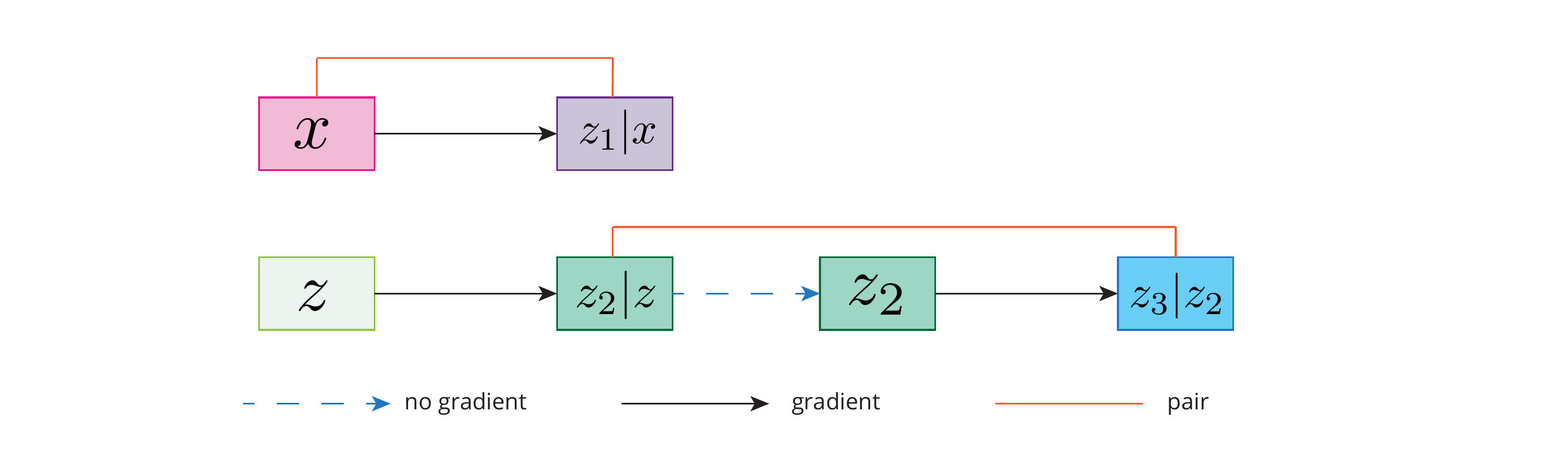}
\caption{Illustration of the generative process for the pairwise discriminator. We block the gradient for $z_2$ to further parallelize the process and improve training speed.}
\label{fig:pairwise-demo}
\end{figure}

\newpage
\section{Additional Experimental Details}
\subsection{Architectures for Generative Model for Images}
\label{sec:architecture}
Code is available at \href{https://github.com/ermongroup/markov-chain-gan}{https://github.com/ermongroup/markov-chain-gan}.

%\se{change code link? there is another link before}

Let `fc $n$, (activation)' denote a fully connected layer with $n$ neurons. Let `conv2d $n$, $k$, $s$, (activation)' denote a convolutional layer with $n$ filters of size $k$ and stride $s$. Let `deconv2d $n$, $k$, $s$, (activation)' denote a transposed convolutional layer with $n$ filters of size $k$ and stride $s$.

We use the following model to generate Figure \ref{fig:mnist} (MNIST).

\begin{table}[H]
\centering
\begin{tabular}{ccc}
\toprule
encoder & decoder & discriminator \\
\midrule
fc 600, lrelu & fc 600, lrelu  & conv2d 64, $4\times 4$, $2\times 2$, relu\\
fc 100, linear & fc 784, sigmoid & conv2d 128, $4\times 4$, $2\times 2$, lrelu \\
& & fc 600, lrelu \\
& & fc 1, linear \\
\bottomrule
\end{tabular}
\end{table}

We use the following model to generate Figure \ref{fig:celeba} (CelebA, top)

\begin{table}[H]
\centering
\begin{tabular}{ccc}
\toprule
encoder & decoder & discriminator \\
\midrule
conv2d 64, $4\times 4$, $2\times 2$, lrelu & fc $16 \times 16 \times 64$, lrelu  & conv2d 64, $4\times 4$, $2\times 2$, relu\\
fc 200, linear & deconv2d 3, $4\times 4$, $2\times 2$, tanh & conv2d 128, $4\times 4$, $2\times 2$, lrelu \\
& & conv2d 256, $4\times 4$, $2\times 2$, lrelu \\
& & fc 1, linear \\
\bottomrule
\end{tabular}
\end{table}

For the bottom figure in Figure \ref{fig:celeba}, we add a residual connection such that the input to the second layer of the decoder is the sum of the outputs from the first layers of the decoder and encoder (both have shape $16 \times 16 \times 64$); we add a highway connection from input image to the output of the decoder:
$$
\bar{x} = \alpha x + (1 - \alpha) \hat{x}
$$
where $\bar{x}$ is the output of the function, $\hat{x}$ is the output of the decoder, and $\alpha$ is an additional transposed convolutional output layer with sigmoid activation that has the same dimension as $\hat{x}$.

We use the following model to generate Figure \ref{fig:pairwise} (CelebA, pairwise):
\begin{table}[H]
\centering
\begin{tabular}{ccc}
\toprule
encoder & decoder & discriminator \\
\midrule
conv2d 64, $4\times 4$, $2\times 2$, lrelu & fc 1024, relu  & conv2d 64, $4\times 4$, $2\times 2$, relu\\
conv2d 64, $4\times 4$, $2\times 2$ & fc $8 \times 8 \times 128$, relu & conv2d 128, $4\times 4$, $2\times 2$, lrelu \\
fc 1024, lrelu & deconv2d 64, $4\times 4$, $2\times 2$, relu & conv2d 256, $4\times 4$, $2\times 2$, lrelu \\
fc 200 linear & deconv2d 3, $4\times 4$, $2\times 2$, tanh & fc 1, linear \\
\bottomrule
\end{tabular}
\end{table}

For the pairwise discriminator, we double the number of filters in each convolutional layer.
According to \cite{gulrajani2017improved}, we only use batch normalization in the generator for all experiments.

\subsection{Analytic Forms of Energy Functions}
\label{sec:math}
Let $f(x | \mu, \sigma)$ denote the log pdf of $\mc{N}(\mu, \sigma^2)$.

The analytic form of $U(x)$ for \textit{ring} is:
\begin{equation}
U(x) = \frac{(\sqrt{x_1^2 + x_2^2} - 2)^2}{0.32}
\end{equation}

The analytic form of $U(x)$ for \textit{mog2} is:
\begin{equation}
U(x) = f(x | \mu_1, \sigma_1) + f(x | \mu_2, \sigma_2) - \log 2
\end{equation}
where $\mu_1 = [5, 0]$, $\mu_2 = [-5, 0]$, $\sigma_1 = \sigma_2 = [0.5, 0.5]$. 

The analytic form of $U(x)$ for \textit{mog6} is:
\begin{equation}
U(x) = \sum_{i=1}^{6} f(x | \mu_i, \sigma_i) - \log 6
\end{equation}
where $\mu_i = [\sin \frac{i\pi}{3}, \cos \frac{i\pi}{3}]$ and $\sigma_i = [0.5, 0.5]$.

The analytic form of $U(x)$ for \textit{ring5} is:
\begin{equation}
U(x) = \min (u_1, u_2, u_3, u_4, u_5)
\end{equation}
where $u_i = (\sqrt{x_1^2 + x_2^2} - i)^2 / 0.04$.

\subsection{Benchmarking Running Time}
\label{sec:benchmark}
Since the runtime results depends on the type of machine, language, and low-level optimizations, we try to make a fair comparison between HMC and A-NICE-MC on TensorFlow \cite{abadi2016tensorflow}.

Our code is written and executed in TensorFlow 1.0. Due to the optimization of the computation graphs in TensorFlow, the wall-clock time does not seem to be exactly linear in some cases, even when we force the program to use only 1 thread on the CPU. The wall-clock time is affected by 2 aspects, batch size and number of steps. We find that the wall-clock time is relatively linear with respect to the number of steps, and not exactly linear with respect to the batch size. 

Given a fixed number of steps, the wall-clock time is constant when the batch size is lower than a threshold, and then increases approximately linearly. To perform speed benchmarking on the methods, we select the batch size to be the value around the threshold, in order to prevent significant under-estimates of the efficiency.

We found that the graph is much more optimized if the batch size is determined before execution. Therefore, we perform all the benchmarks on the optimized graph where we specify a batch size prior to running the graph. For the energy functions, we use a batch size of 2000; for Bayesian logistic regression we use a batch size of 64.

\subsection{Hyperparameters for the Energy Function Experiments}
\label{sec:en}
For all the experiments, we use same hyperparameters for both A-NICE-MC and HMC. We sample $x_0 \sim \mc{N}(0, I)$ and run the chain for 1000 burn-in steps and evaluate the samples from the next 1000 steps.

For HMC we use 40 leapfrog steps and a step size of 0.1. For A-NICE-MC we consider $f_\theta(x, v)$ with three coupling layers, which updates $v$, $x$ and $v$ respectively. The motivation behind this particular architecture is to ensure that both $x$ and $v$ could affect the updates to $x^\prime$ and $v^\prime$. In each coupling layer, we select the function $m(\cdot)$ to be a one-layer NN with 400 neurons. The discriminator is a three layer MLP with 400 neurons each.
Similar to the settings in Section \ref{sec:images}, we use the gradient penalty method in \cite{gulrajani2017improved} to train our model. 

For bootstrapping, we first collect samples by running the NICE proposal over the untrained $f_\theta$, and for every 500 iterations we replace half of the samples with samples from the latest trained model. All the models are trained with AdaM \cite{kingma2014adam} for 20000 iterations with $B = 4$, $M = 2$, batch size of 32 and learning rate of $10^{-4}$.

\subsection{Hyperparameters for the Bayesian Logistic Regression Experiments}
\label{sec:blr}
For HMC we tuned the step size parameter to achieve the best ESS possible on each dataset, which is $0.005$ for \textit{german}, $0.01$ for \textit{heart} and $0.0115$ for \textit{australian} (HMC performance on \textit{australian} is extremely sensitive to the step size). For A-NICE-MC we consider $f(x, v)$ with three coupling layers, which updates $v$, $x$ and $v$ respectively; we set $v$ to have 50 dimensions in all the experiments. $m(\cdot)$ is a one-layer NN with 400 neurons for the top and bottom coupling layer, and a two-layer NN with 400 neurons each for the middle layer. The discriminator is a three layer MLP with 800 neurons each. We use the same training and bootstrapping strategy as in Appendix \ref{sec:en}. All the models are trained with AdaM for 20000 iterations with $B = 16$, $M = 2$, batch size of 32 and learning rate of $5 \times 10^{-4}$.

\subsection{Architecture Details}
The following figure illustrates the architecture details of $f_\theta(x, v)$ for A-NICE-MC experiments. We do not use batch normalization (or other normalization techniques), since it slows the execution of the network and does not provide much ESS improvement.

\begin{figure}[H]
\centering
\begin{subfigure}{0.48\textwidth}
\centering
\includegraphics[width=0.8\textwidth]{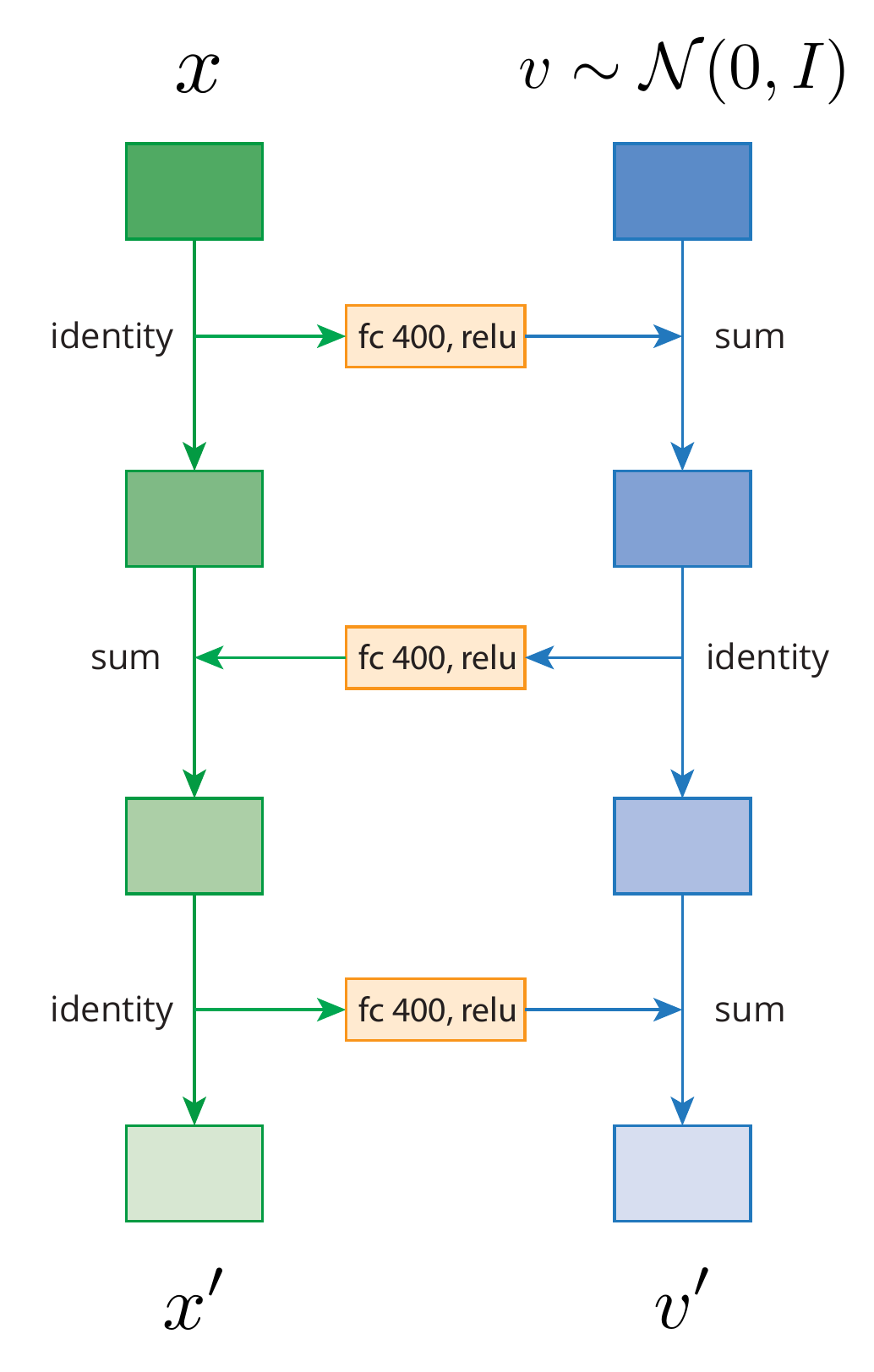}
\caption{NICE architecture for energy functions.}
\end{subfigure}
~
\begin{subfigure}{0.48\textwidth}
\centering
\includegraphics[width=0.8\textwidth]{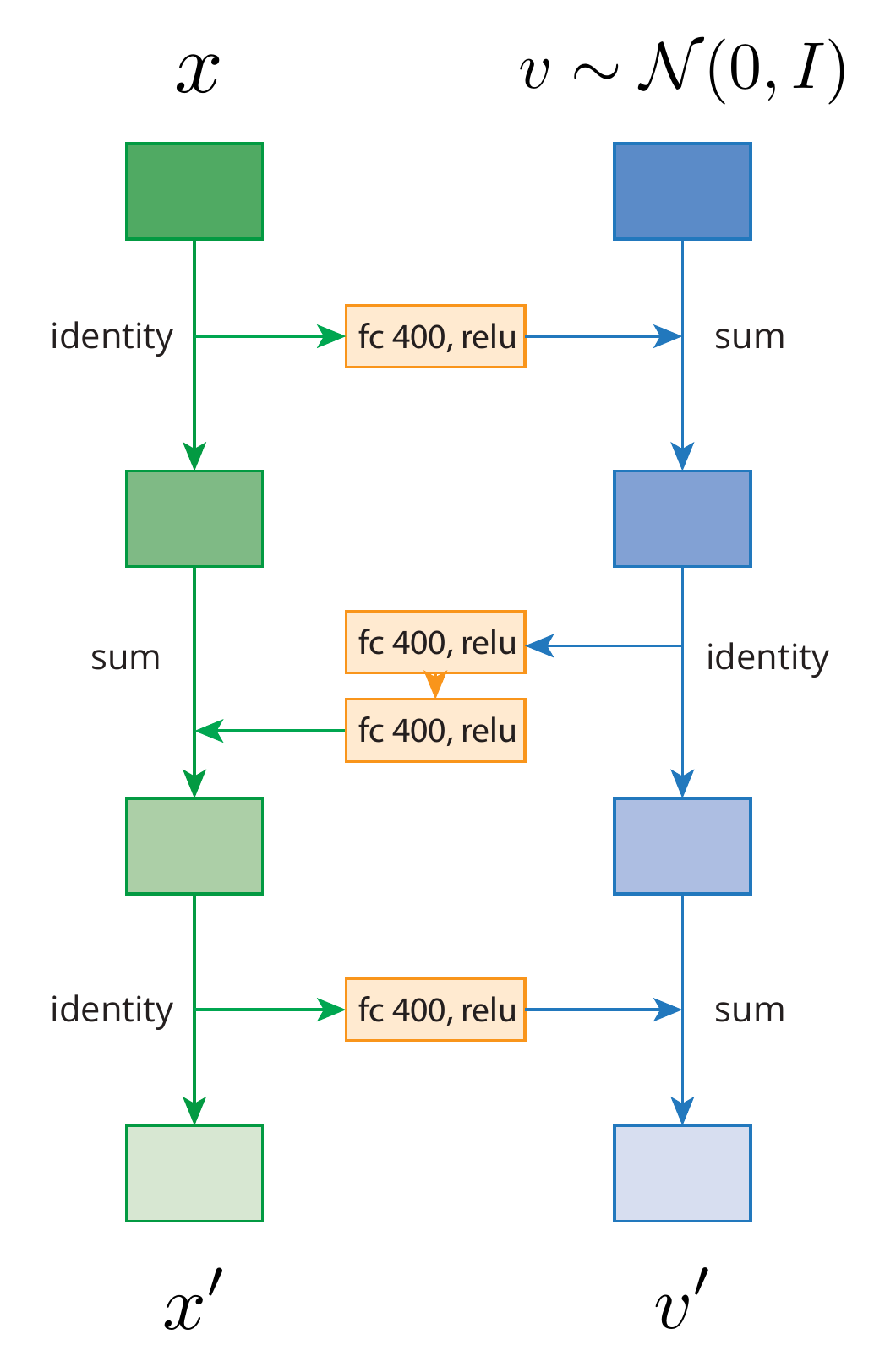}
\caption{NICE architecture for Bayesian logistic regression.}
\end{subfigure}
\end{figure}

\newpage
\section{Extended Images}
\label{sec:extended}
We only displayed a small number of images in the main text due to limited space. Here we include the extended version of images for our image generation experiments.
\subsection{Extended Images for Figure \ref{fig:mnist}}
\label{sec:fig:mnist}
\begin{figure}[H]
\centering
\includegraphics[width=\textwidth]{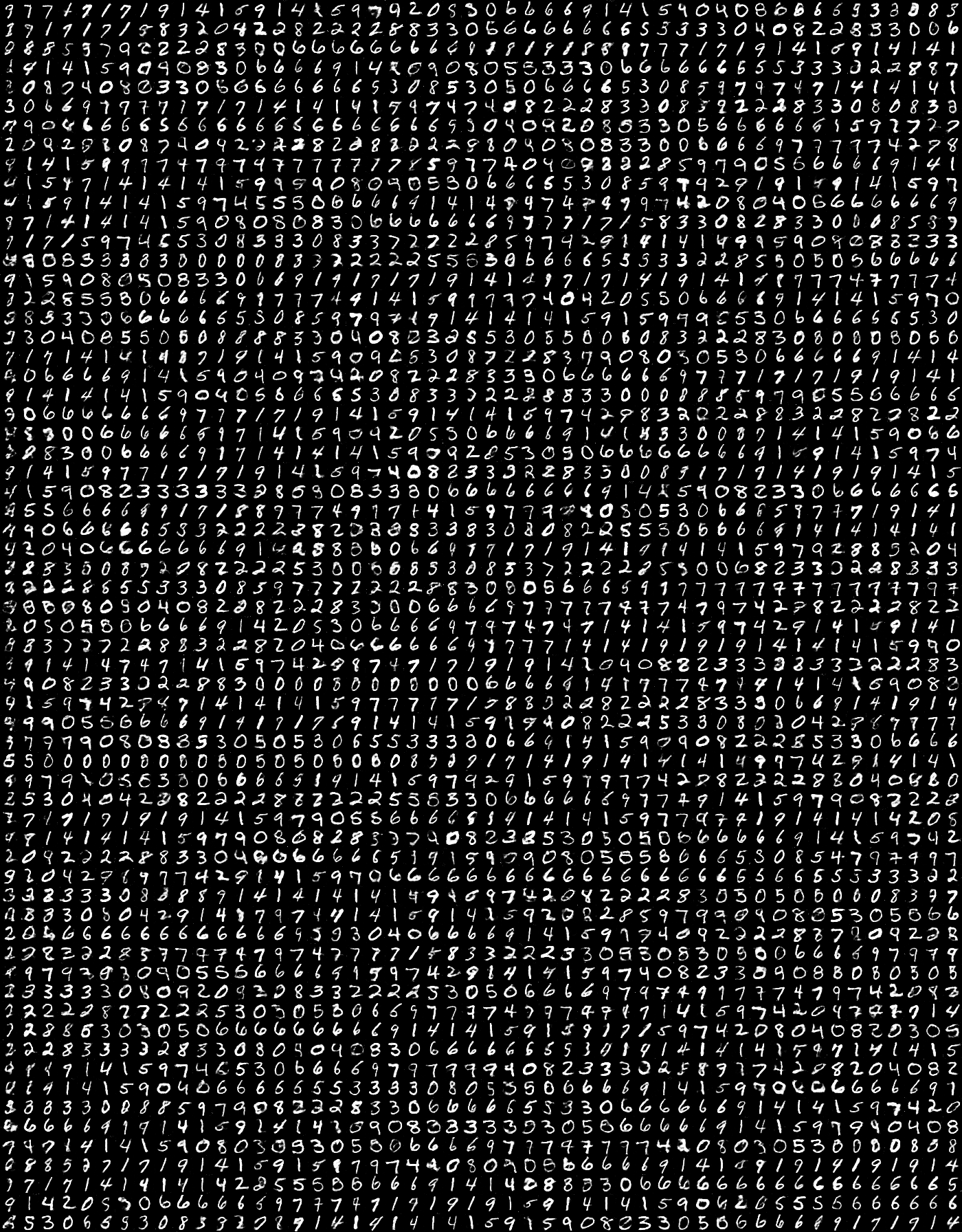}
\caption{Samples from $\pi_1$ to $\pi_{50}$ from a model trained on the MNIST dataset. Each row are samples from the same chain.}
\end{figure}

\subsection{Extended Images for Figure \ref{fig:celeba}}
\label{sec:fig:celeba}
The following models are trained with the original MGAN objective (without pairwise discriminator).
\begin{figure}[H]
\centering
\includegraphics[width=\textwidth]{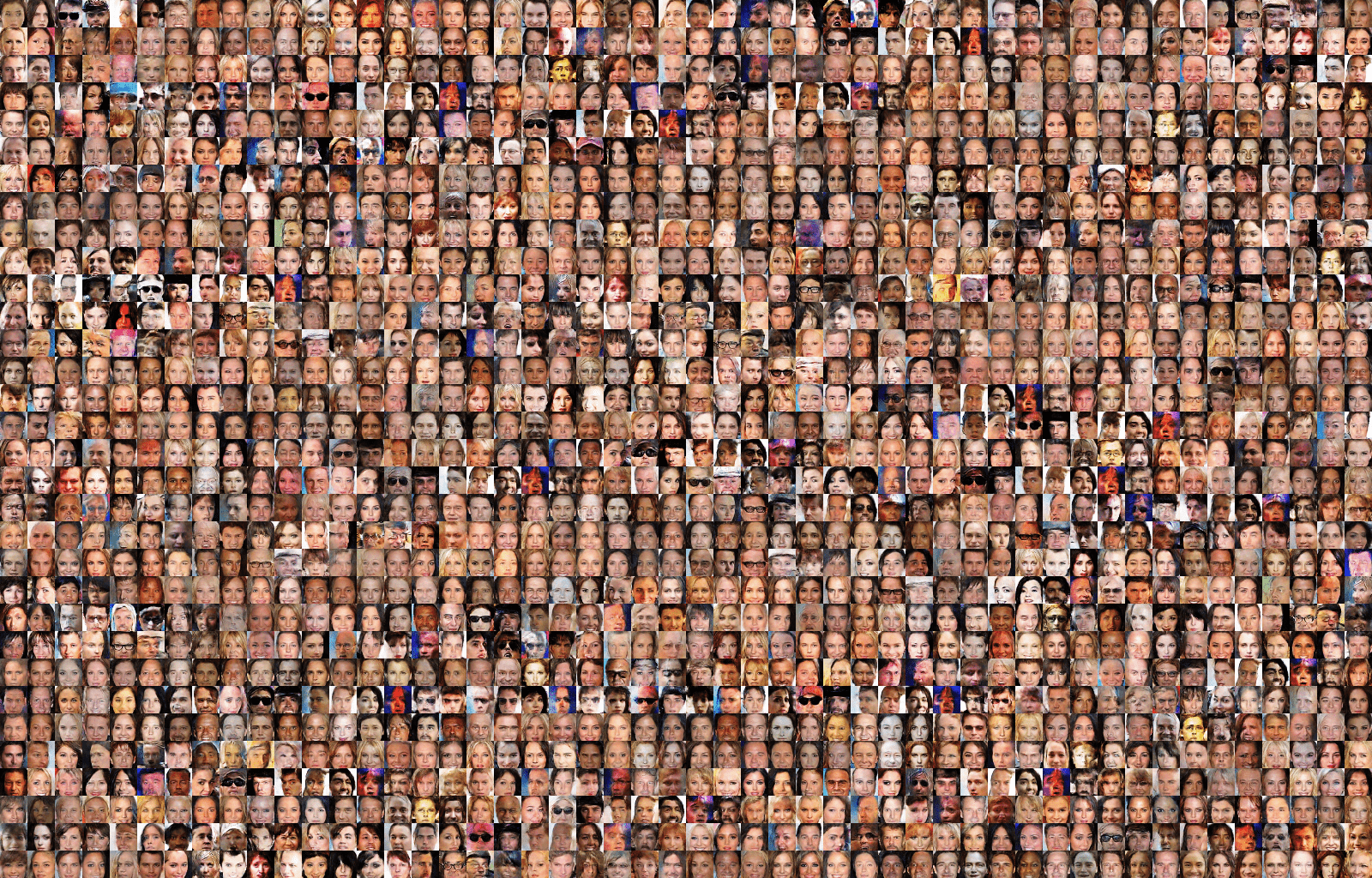}
\caption{Samples from $\pi_1$ to $\pi_{50}$ from a model trained on the CelebA dataset. Each row are samples from the same chain.}
\end{figure}

\begin{figure}[H]
\centering
\includegraphics[width=\textwidth]{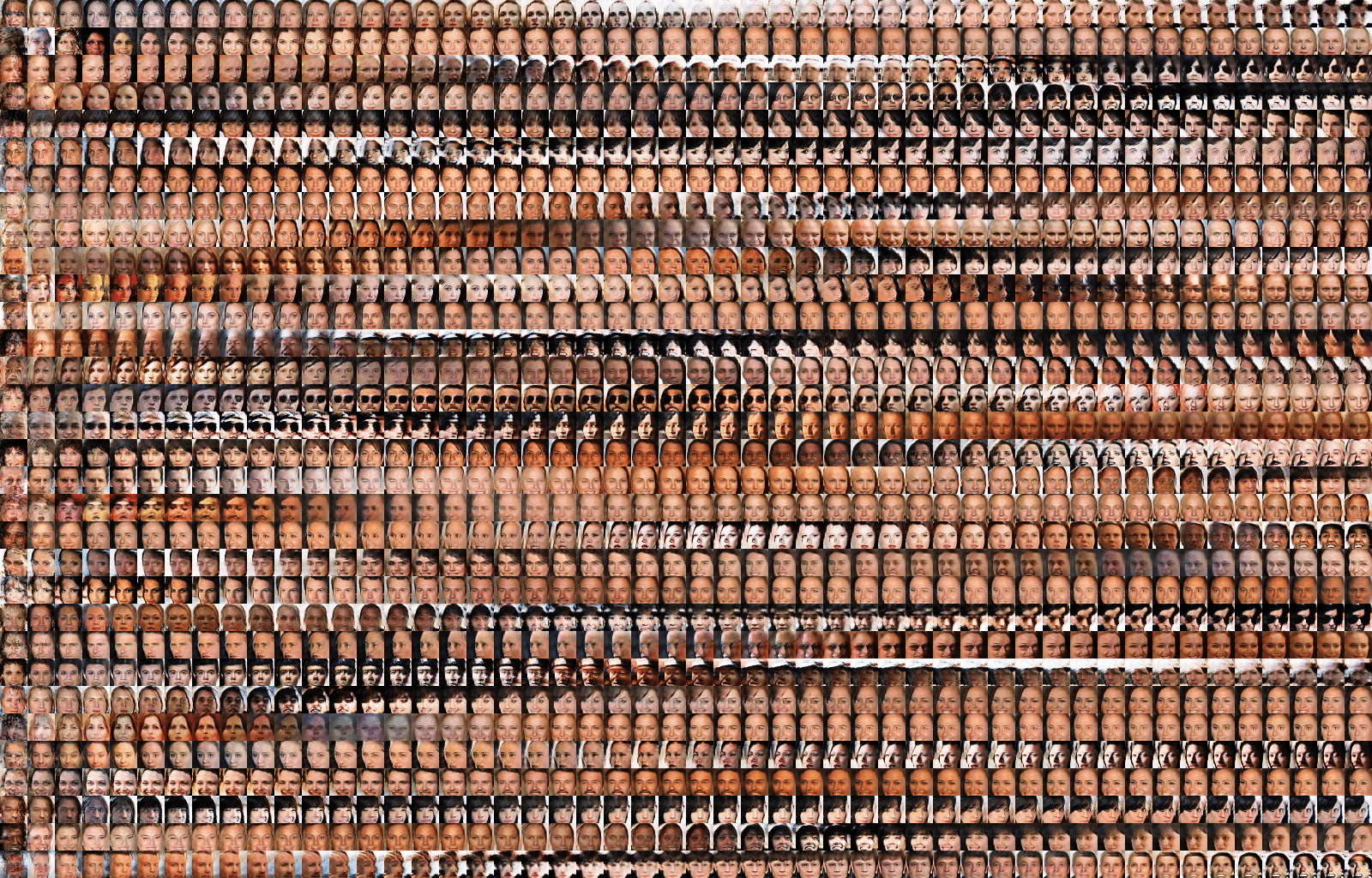}
\caption{Samples from $\pi_1$ to $\pi_{50}$ from a model trained on the CelebA dataset, where the model has shortcut connections. Each row are samples from the same chain.}
\end{figure}

\subsection{Extended Images for Figure \ref{fig:pairwise}}
\label{sec:fig:pairwise}
The following images are trained on the same model with shortcut connections.
\begin{figure}[H]
\centering
\includegraphics[width=\textwidth]{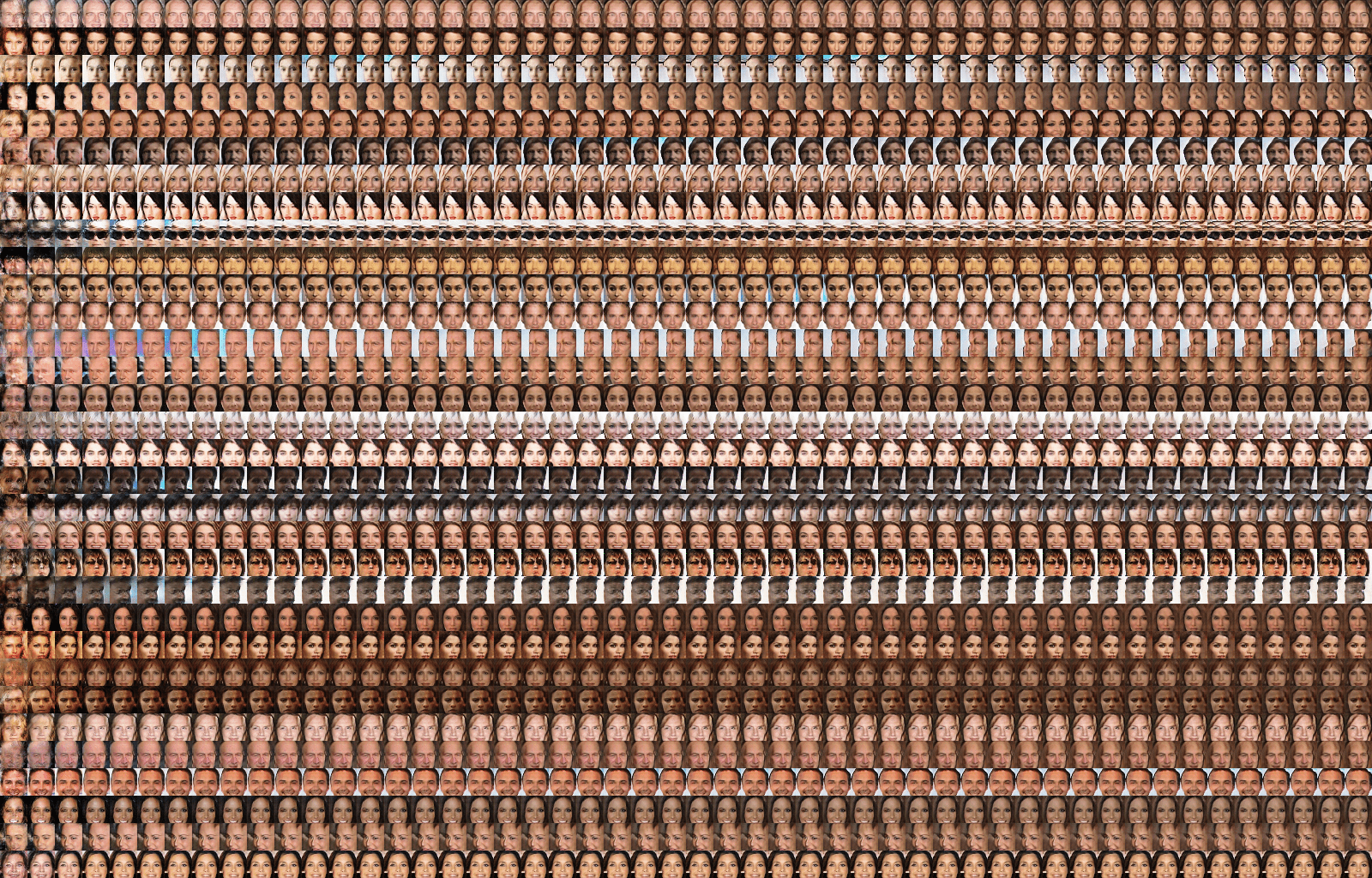}
\caption{Samples from $\pi_1$ to $\pi_{50}$ from a model trained on the CelebA dataset without pairwise discriminator. Each row are samples from the same chain.}
\end{figure}

\begin{figure}[H]
\centering
\includegraphics[width=\textwidth]{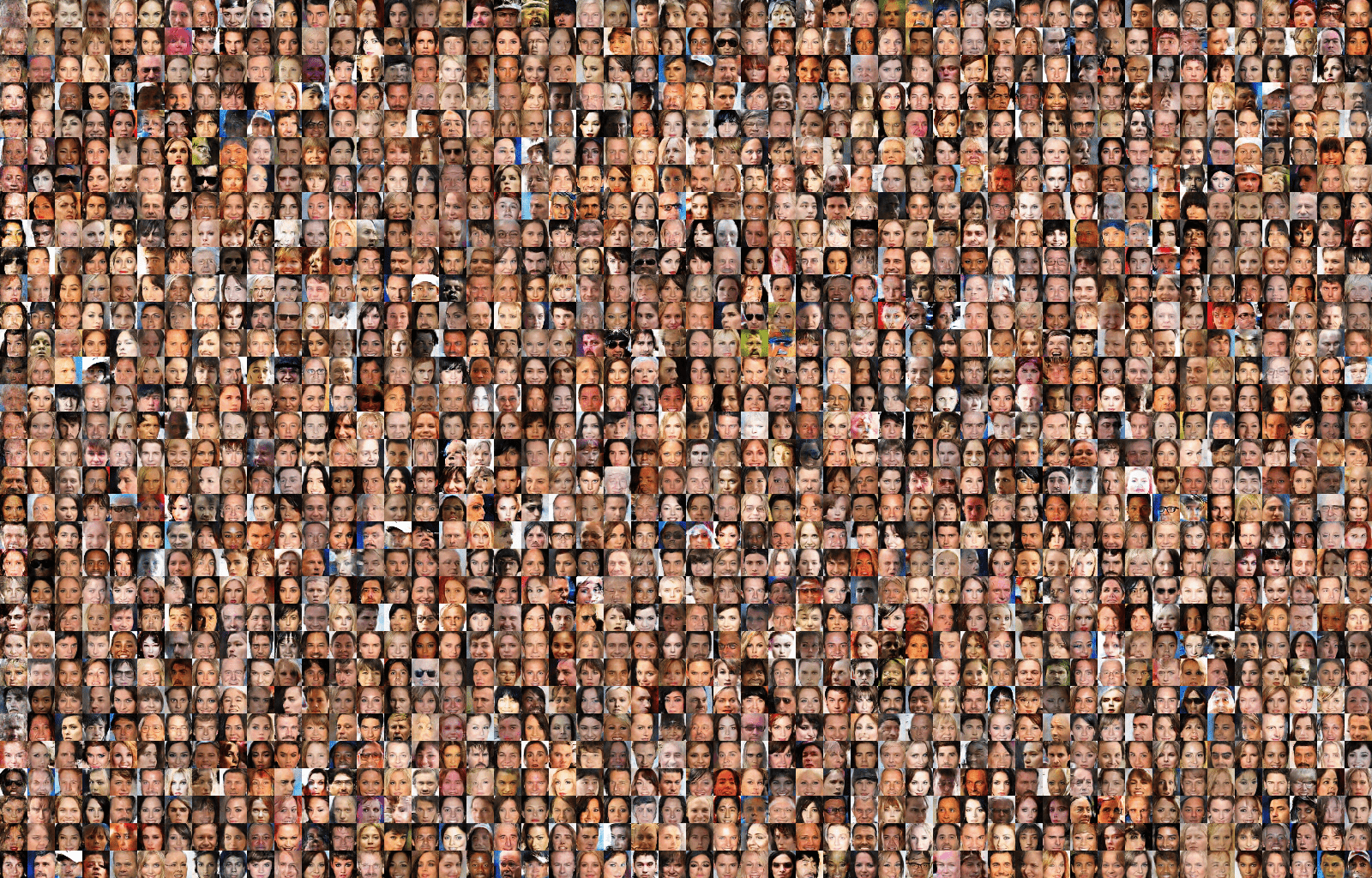}
\caption{Samples from $\pi_1$ to $\pi_{50}$ from a model trained on the CelebA dataset with pairwise discriminator. Each row are samples from the same chain.}
\end{figure}

\end{document}